%% file: main_arxiv.tex
\documentclass{article}

\usepackage[final]{corl_2026_arxiv} %
\usepackage{amsmath,amssymb,amsthm}
\usepackage{mathtools}

\DeclarePairedDelimiterX{\infdivx}[2]{(}{)}{%
  #1\;\delimsize\|\;#2%
}
\usepackage{authblk}
\usepackage{bm}
\usepackage{xspace}
\usepackage{graphicx}
\usepackage{wrapfig}
\usepackage{subcaption}
\usepackage{enumitem}
\usepackage{booktabs}
\usepackage{cleveref}

\newtheorem{propositionmain}{Proposition}
\newtheorem{definitionmain}{Definition}

\usepackage{tcolorbox}
\tcbuselibrary{theorems}
\usepackage{thmtools}

\definecolor{ourcolor}{HTML}{00A89D}
\definecolor{ourpurple}{HTML}{CB297B}
\definecolor{ourblue}{HTML}{0076BA}
\definecolor{ourmiddle}{HTML}{66509B}
\definecolor{ourlightblue}{HTML}{F5FAFE}
\definecolor{ourlightmiddle}{HTML}{F7F6F9}

\declaretheoremstyle[
    headfont=\bfseries,
    bodyfont=\normalfont,
    spaceabove=10pt, spacebelow=10pt,
    headpunct={},
    postheadspace=1em,
    mdframed={
        backgroundcolor=ourcolor!5,
        linecolor=ourcolor!50!black,
        linewidth=1pt,
        roundcorner=4pt
    }
]{colorboxed}

\declaretheoremstyle[
    headfont=\bfseries,
    bodyfont=\normalfont,
    spaceabove=10pt, spacebelow=10pt,
    headpunct={},
    postheadspace=1em,
    mdframed={
        backgroundcolor=ourblue!5,
        linecolor=ourblue!50!black,
        linewidth=1pt,
        roundcorner=4pt
    }
]{bluecolorboxed}

\declaretheoremstyle[
    headfont=\bfseries,
    bodyfont=\normalfont,
    spaceabove=10pt, spacebelow=10pt,
    headpunct={},
    postheadspace=1em,
    mdframed={
        backgroundcolor=ourpurple!5,
        linecolor=ourpurple!50!black,
        linewidth=1pt,
        roundcorner=4pt
    }
]{purplecolorboxed}

\declaretheorem[name=Theorem,style=colorboxed]{theorem}
\declaretheorem[name=Lemma,style=purplecolorboxed]{lemma}

\declaretheorem[name=Definition,style=bluecolorboxed]{definition}

\declaretheorem[name=Assumption,style=bluecolorboxed]{assumption}

\crefname{assumption}{Assumption}{Assumption}

\title{Learning to Act While Waiting: RL Finetuning of Generalist Robot Policies Under Inference Latency}

\newcommand{\suppdetail}[1]{}

\newcommand{\method}{\textsc{Arli}\xspace}
\newcommand{\loose}{\looseness=-1}

\newtoggle{arxiv}
\toggletrue{arxiv}

\author[1]{Brian Zhu$^*$}
\author[1]{Momen Khalil$^*$}
\author[2]{E Harrison$^*$}
\author[1]{Emanuele Poggi$^*$}
\author[1]{Philipp Schmitt}
\author[1]{Bernd Kast}
\author[1]{Philine Meister}
\author[2]{Pranav Atreya}
\author[2]{Qiyang Li}
\author[1]{Finn Ferchau}
\author[1]{Cesar Colmenero}
\author[1]{Yash Shahapurkar}
\author[1]{Gokul Narayanan}
\author[1]{Melih Erdogan}
\author[1]{Kai Wurm}
\author[1]{Georg von Wichert}
\author[3,4,2]{Oier Mees}
\author[1]{Eugen Solowjow}
\author[2]{Andrew Wagenmaker}
\author[2]{Sergey Levine}
\affil{
    Siemens
    \qquad
    $^2$ UC Berkeley
    \qquad
    $^3$ Microsoft
    \qquad
    $^4$ ETH Zurich
}

\begin{document}

\maketitle
\newcommand\blfootnote[1]{%
  \begingroup
  \renewcommand\thefootnote{}\footnote{#1}%
  \addtocounter{footnote}{-1}%
  \endgroup
}
\renewcommand{\subsectionautorefname}{Subsection}
\blfootnote{$^*$Equal Contribution}
\vspace{-5em}
\begin{abstract}
  While reinforcement learning (RL) allows generalist robot policies to continually improve during deployment, the large model size of modern generalist policies, such as VLAs, poses a fundamental obstacle to effective RL improvement. In particular, their severe inference latency---which can lead to pauses or jerky movements---can alter the effective environment dynamics and, if not correctly accounted for, break the Markov assumption that RL relies on, causing standard RL algorithms to fail completely. In this work, we introduce a latency-aware framework, \emph{\textbf{A}synchronous \textbf{RL} with \textbf{I}ntermediate Information} (\method), that enables RL-based improvement of generalist policies under inference delays. 
  Our framework builds on asynchronous inference approaches, which interleave action generation with execution to hide latency, and addresses its incompatibility with RL by providing a low-latency RL policy design that maximizes reactivity within the inference window through two contributions: state augmentations that restore near-Markovian structure by incorporating committed actions and a mid-inference observation. 
  We evaluate our approach across simulated and real-world manipulation tasks, and find that it enables effective finetuning under inference delays where standard RL fails entirely, even matching or exceeding the performance of standard RL in idealized no-latency settings.  
  Website: \texttt{\href{https://async-rl-intermediate-information.github.io/}{https://async-rl-intermediate-information.github.io/}}

\end{abstract}


\input{body/introduction}

\input{body/method}

\input{body/experiments}

\input{body/limitations}

\clearpage
\acknowledgments{This research was partially supported by Siemens and ONR N00014-25-1-2060.}

\bibliography{arxiv_bib}

\newpage
\appendix
\input{body/appendix}
\end{document}

%% file: body/introduction.tex
\newcommand{\arli}{\textsc{Arli}\xspace}

\begin{figure}[ht]
    \centering
    \includegraphics[width=0.94\linewidth]
    {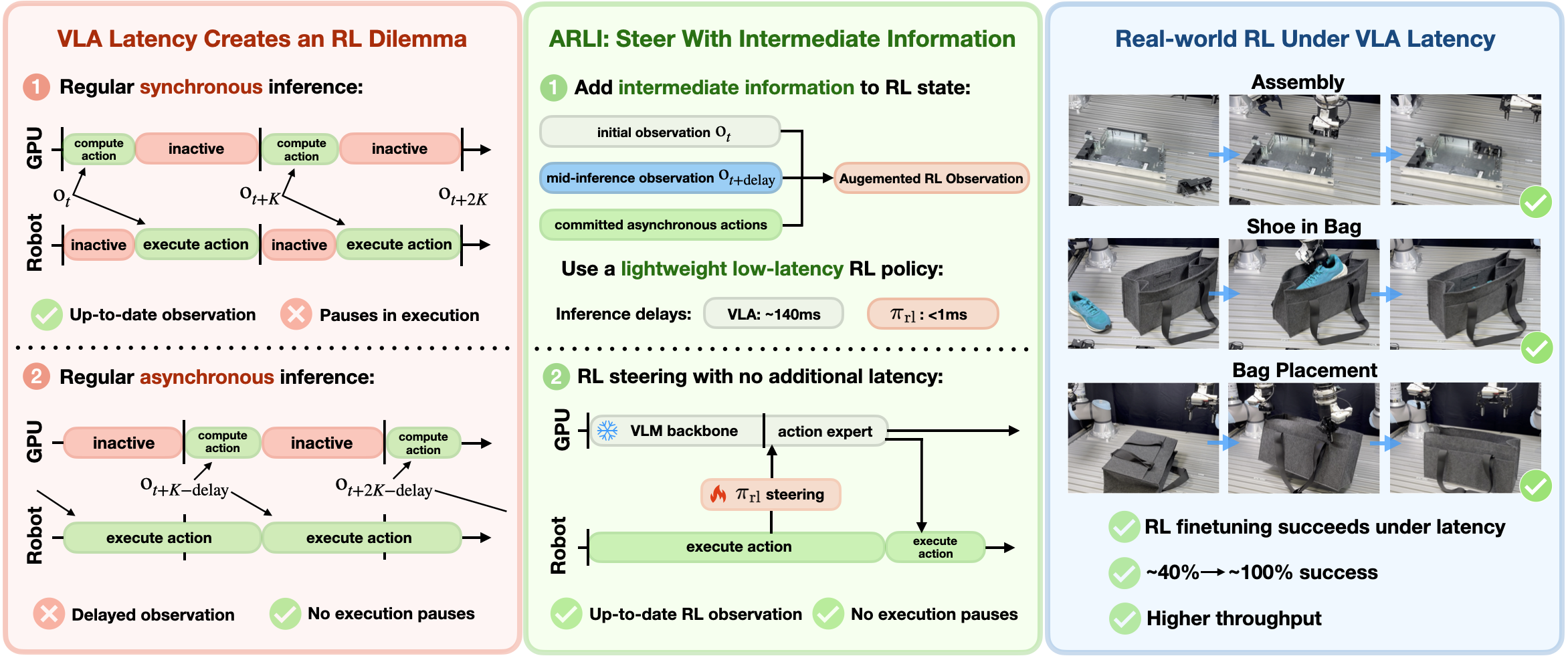}
    \caption{\textbf{Enabling RL improvement of generalist robot policies under inference latency:} Generalist policies incur significant action inference latency delays, either forcing pauses during inference or relying on stale observations, making RL improvement brittle. \arli uses intermediate actions and updated asynchronous observations to steer the next action chunk, enabling effective real-world improvement under asynchronous inference, and real-world VLA improvement.}
    \label{fig:teaser_fig}
    \end{figure}

\vspace{-1.25em}
\section{Introduction}
\vspace{-0.75em}

Enabling robots to learn and improve through experience collected in deployment has proven a critical step in achieving highly performant robotic behaviors. In particular, with the advent of \emph{generalist} robot policies---policies trained on diverse, large-scale datasets to perform a variety of tasks \cite{brohan2022rt,gu2023rt,octo_2023,Doshi24-crossformer,kim2024openvla,black2024pi_0,bjorck2025gr00t,team2025gemini,liu2025rdt,intelligence2025pi_,zha2026lap}---online improvement and adaptation approaches have enabled significantly more reliable and high-precision behaviors than are typically achievable via offline training alone ~\cite{wagenmaker2025steering,amin2025pi,li2025gr,sun2026prior,murray2026flowdagger,chen2026dgac}. While significant progress has been made developing RL-based approaches for improving generalist policies, the most promising results remain limited to constrained, lab-like settings. As generalist policies are increasingly deployed in real-world settings, enabling RL improvement under practical deployment constraints is critical to achieving effective real-world robotic control.

\iftoggle{arxiv}{
A key constraint present in the deployment of generalist policies is \emph{inference latency}. Generalist policies, such as vision-language-action models (VLAs), are typically very large---commonly billions of parameters \cite{kim2024openvla,black2024pi_0,bjorck2025gr00t,zha2026lap}---and generating an action with a generalist policy can take a non-trivial amount of time. For example, on a standard consumer-grade GPU, running a single action generation step on $\pi_0$, a commonly used open-source VLA, takes approximately 100 milliseconds \cite{black2024pi_0}, while other open-source models can take over 300 milliseconds per action generation \cite{kim2024openvla,kim2025fine}. Such inference delays can lead to pauses or jerky movements in deployment, altering the effective dynamics of the environment and leading to poor policy performance, especially in settings that require high reactivity.
While a growing body of work has sought to address challenges introduced by policy latency \cite{black2025training,shukor2025smolvla,black2026real}, these works typically focus simply on policy \emph{deployment}, and do not investigate how latency affects the ability of RL-based approaches to learn effectively.

In this work we seek to close this gap and enable effective RL improvement of generalist policies under inference latency constraints. As a starting point, we build on several recent works that propose \emph{asynchronous inference} strategies, where the generalist policy begins computing the next action while the current action is still being executed \cite{shukor2025smolvla,black2026real}. While asynchronous inference, if correctly instantiated, has been shown to significantly mitigate the effects of inference latency, it comes at a cost: if we begin policy inference while the previous action is being executed, we cannot condition on the actual environment state we will be at when we play this action, since we have not yet reached this state. As a result, naively combining RL improvement with asynchronous inference leads to a \emph{non-Markovian effective state}, and causes standard RL approaches to fail to learn.
}{
A key challenge in deploying generalist policies is \emph{inference latency}: large vision-language-action models (VLAs), often containing billions of parameters \cite{kim2024openvla,black2024pi_0,bjorck2025gr00t,zha2026lap,gao2026steervla}, can require 100--300+ ms to generate actions on consumer GPUs \cite{black2024pi_0,kim2024openvla,kim2025fine}. Such delays can cause pauses or jerky motions that alter the effective environment dynamics and degrade performance, especially in settings that require high reactivity. Although prior work has proposed \emph{asynchronous inference} strategies that compute the next action while the current action is being executed to mitigate these effects \cite{shukor2025smolvla,black2026real}, these approaches focus primarily on deployment and introduce a key challenge for RL: since inference begins before the current action finishes executing, the policy cannot condition on the actual state where the action will be played. Thus, naively combining RL with asynchronous inference produces a \emph{non-Markovian effective state}, causing standard RL methods to fail.
}

To address this challenge and enable RL finetuning with asynchronous inference, we develop \arli, or \emph{\textbf{A}synchronous \textbf{RL} with \textbf{I}ntermediate Information}, that overcomes the non-Markovian state with careful state augmentation. In particular, \arli is inspired by two key insights. First, by including the \emph{intermediate actions} in the state, we can give the RL policy significantly more predictive information about future states, allowing it to more accurately finetune the generalist policy's actions.
Second, in many cases, \emph{RL inference is required at a later stage than generalist policy inference}, and we can thus integrate the RL corrections at a later time than when we begin generalist policy inference. Given this, the RL policy can be conditioned on a more up-to-date \emph{intermediate state} than what the generalist policy has access to, mitigating the effect of the asynchronous inference delay.

We show that \method leads to effective improvement of generalist policies under practical latency constraints. In particular, we evaluate \method on several simulated tasks where high reactivity is required, and find that it yields significantly more efficient RL improvement, in many cases enabling improvement when naive approaches to combining RL with asynchronous inference fail. We then show that this holds in three challenging real-world tasks on a bimanual UR5e robot, finding that \method enables real-world policy improvement under asynchronous inference in regimes where both naively combining RL with asynchronous inference, and attempting to use synchronous inference, fail.\loose

\vspace{-0.25em}
\section{Related Work}
\vspace{-0.25em}

\textbf{Robot policy deployment under latency constraints.}
With the advent of ``generalist'' robotic foundation models, inference delays have become a central concern in policy deployment, as the forward passes of such policies  routinely exceed the controller's sampling period~\cite{rt22023arxiv,kim2024openvla,black2024pi_0,pai25,ye2026world,hou2026world}. Action chunking~\cite{zhao2023learning} amortizes inference cost by producing $k$ actions per forward pass, and is now standard in modern robot policies~\cite{chi2023diffusion,dasari2024ingredients,octo_2023,black2024pi_0,pertsch25-fast,intelligence2025pi_}. Action chunking only partially addresses the latency problem, however, as synchronous execution incurs periodic pauses between chunks, yet naive asynchronous execution introduces discontinuities at chunk boundaries~\cite{black2025training,black2026real}. Recent work mitigates these discontinuities through temporal ensembling~\cite{zhao2023learning}, rejection sampling~\cite{liu2025bidirectional}, or diffusion inpainting~\cite{black2026real}. A separate line of sim-to-real work injects delays into simulation during training to obtain delay-robust policies~\cite{tan2018sim,hwangbo2019learning}. All of these methods, however, target deployment of a fixed pretrained policy and do not consider online adaptation.\loose

\textbf{RL under latency constraints.}
Prior work has studied reinforcement learning and control under delayed information, actions, and rewards, from early formulations of delayed closed-loop control and delayed-state/action MDPs \cite{altman2003closed,katsikopoulos2003markov,walsh2009learning,schuitema2010control} to more recent methods for stochastic delays and real-time decision making \cite{firoiu2018human,bouteiller2021reinforcement}. Closely related work also considers concurrent control, where agents continue acting while new decisions are being computed \cite{xiao2020thinking}. These works largely consider generic RL settings or train policies from scratch, and do not address the action-chunked asynchronous regime of modern diffusion and VLA policies, the setting we target in this work.

\textbf{RL for robotics.}
Reinforcement learning has played a key role enabling highly performant policies for robotic control, spanning domains from locomotion \cite{smith2022walk,smith2022legged} to manipulation \cite{zhu2020ingredients,luo2024serl,mendonca2024continuously,rosete2022tacorl,luo2025precise}. 
A common strategy to improve the sample efficiency of RL is to rely on simulators, transferring sim-learned policies to the real world \cite{cutler2014reinforcement,rajeswaran2016epopt,tobin2017domain,peng2018sim,chebotar2019closing,lee2020learning,kumar2021rma}, yet such approaches are often limited by the sim-to-real gap \cite{wagenmaker2024overcoming}. With the advent of generalist robot policies, significant focus has been devoted to enabling improvement and adaptation of such policies \cite{zhang2024grape,mark2024policy, nakamoto2024steering, chen2025conrft, hu2025flare, ankile2025imitation, wagenmaker2025steering, dong2025matters, dong2025expo, xiao2025self, xu2026rl, sun2026prior, amin2025pi,  guo2025improving, lu2025vla, li2025gr, myers2024policy, liu2025can}. In this work we build on \cite{wagenmaker2025steering}, which enables RL improvement of pretrained policies by steering the input noise to the policy's denoising process (in the case when the pretrained policy is a diffusion or flow model). However, none of these works explicitly deal with enabling RL finetuning under latency constraints.

%% file: body/method.tex
\newcommand{\cS}{\mathcal{S}}
\newcommand{\cA}{\mathcal{A}}
\newcommand{\cM}{\mathcal{M}}
\newcommand{\cN}{\mathcal{N}}
\newcommand{\R}{\mathbb{R}}
\newcommand{\Exp}{\mathbb{E}}
\newcommand{\pipt}{\pi_{\mathrm{pt}}}
\newcommand{\piptvlm}{\pi_{\mathrm{pt}}^{\mathrm{vlm}}}
\newcommand{\piptae}{\pi_{\mathrm{pt}}^{\mathrm{ae}}}
\newcommand{\ba}{\bm{a}}
\newcommand{\tdelay}{t_{\mathrm{delay}}}
\newcommand{\tdelayrl}{t_{\mathrm{delay}}^{\mathrm{rl}}}
\newcommand{\Pinit}{P_{\mathrm{init}}}
\newcommand{\pirl}{\pi_{\mathrm{rl}}}
\newcommand{\srl}{s^{\mathrm{rl}}}

\newcommand{\cE}{\mathcal{E}}
\newcommand{\Prob}{\mathbb{P}}
\newcommand{\rtc}{\textsc{Rtc}\xspace}
\newcommand{\dsrl}{\textsc{Dsrl}\xspace}
\newcommand{\Deltaact}{\Delta_{a}}

\section{Preliminaries}\label{sec:prelim}

We consider decision-making in Markov decision processes (MDPs). An MDP is denoted by a tuple $\cM = (\cS,\cA,P,\Pinit,r,\gamma)$ where $\cS$ is the state space, $\cA$ the action space, $P : \cS \times \cA \rightarrow \triangle_{\cS}$ the transition kernel, $\Pinit \in \triangle_{\cS}$ the initial state distribution, $r : \cS \rightarrow \R$ the reward, and $\gamma \in [0,1]$ the discount factor. Interaction with the MDP proceeds in episodes. First the environment samples an initial state $s_1 \sim \Pinit$, the agent then selects some action $a_1 \in \cA$, the environment transitions to state $s_2 \sim P(s_1,a_1)$, and so forth, proceeding until the episode terminates. A policy denotes a mapping from states to actions, $\pi : \cS \rightarrow \triangle_{\cA}$. For a policy $\pi$ and state $s \in \cS$, the value function quantifies the expected discounted reward, $V^\pi(s) := \Exp^\pi[\sum_{t=1}^\infty \gamma^{t-1} r(s_t) \mid s_1 = s]$, where $\Exp^\pi[\cdot]$ denotes the expectation over trajectories induced by executing $\pi$ on $\cM$. The value of a policy, $V(\pi) := \Exp_{s_1 \sim \Pinit}[V^\pi(s_1)]$, denotes its expected reward over trajectories. The typical goal of RL, and the goal we consider in this work, is to learn a policy that maximizes $V(\pi)$.

\iftoggle{arxiv}{
\textbf{Pretrained policies and inference delays.}
In this work, we assume we are given an initial pretrained policy $\pipt$, for example, a policy trained via behavioral cloning on human demonstrations. We make several assumptions on $\pipt$. 
First, we assume that instead of predicting a single action at each step, $\pipt$ predicts an \emph{action chunk}, a sequence of $k$ actions, $A = (a_1,\ldots,a_k)$. Typically, after producing $A$ the policy executes the first $n$ steps in the chunk before recomputing a new chunk. Second, we assume that $\pipt$ is a diffusion or flow policy \cite{chi2023diffusion}. As such, $\pipt$ not only takes a state $s$ as input, but a random noise $w$, typically sampled from $\mathcal{N}(0, I)$. We note that both of these assumptions are standard design choices in modern robot learning, so this is not a major restriction \cite{octo_2023,black2024pi_0,bjorck2025gr00t,team2025gemini,liu2025rdt,intelligence2025pi_,zha2026lap}.
When $\pipt$ is a VLA or other large model, simply computing $A_t = \pipt(s_t, w_t)$ for $w_t \sim \cN(0,I)$ can take a significant amount of time. Formally, we assume that we can generate an action from $\pipt$ but that this generation will require $\tdelay$ steps (for simplicity, we assume $\tdelay < k \cdot \Deltaact$, where $\Deltaact$ is the time interval we execute each action). We will overload notation somewhat and let $\tdelay$ refer to both the inference delay (in seconds) as well as the number of environment ``steps'' it occupies (where each environment step is an interval of length $\Deltaact$).
}{
\textbf{Pretrained policies and inference delays.}
We assume access to a pretrained policy $\pipt$ (e.g. a VLA) that predicts an action chunk $A=(a_1,\ldots,a_k)$ rather than a single action. After predicting a chunk, the policy executes the first $n$ actions before recomputing a new chunk. We further assume $\pipt$ is a diffusion or flow policy \cite{chi2023diffusion}, meaning it conditions on both the state $s$ and a noise input $w \sim \cN(0,I)$; both assumptions are standard in modern robot learning \cite{team2024octo,black2024pi_0,bjorck2025gr00t,team2025gemini,liu2025rdt,intelligence2025pi_,zha2026lap}. When $\pipt$ is a VLA or similarly large model, generating $A_t = \pipt(s_t, w_t)$ can incur substantial inference latency. We therefore assume action generation requires $\tdelay$ steps, with $\tdelay < k \cdot \Deltaact$, where $\Deltaact$ is the execution interval between actions (the environment ``step''). We will overload notation somewhat and refer to $\tdelay$ as both the inference delay (in seconds) and the number of environment steps it occupies.\loose}

\textbf{Asynchronous inference.}
To reduce effective latency, recent works have suggested computing fresh actions \emph{asynchronously while executing previous action chunks} \cite{black2026real, shukor2025smolvla}. For example, assume at step $t-k$ we have computed action chunk $A_{t-k}$, and are beginning to take these actions in our environment. Instead of waiting until step $t$ to compute $A_{t}$, we can begin inference at some step $t' < t - \tdelay$, such that the next action chunk $A_{t'}$ will be ready to execute as soon as $A_{t-k}$ has been executed. When step $t$ is reached and the action chunk $A_{t-k}$ has finished, we execute $A_{t'}$ starting at step $\tdelay$---the step in $A_{t'}$ that would correspond to timestep $t$. The challenge, of course, is that at step $t'$ we have access only to state $s_{t'}$, and so must produce actions based on $s_{t'}$, leading to potentially out-of-date information and poor action predictions.

\textbf{Real-Time Chunking (\rtc).}
To mitigate the effects of this inference delay, \cite{black2026real} proposes \emph{real-time chunking} (\rtc), which seeks to ensure continuity between each generated action chunk. Specifically, \rtc incorporates  actions $[A_{t-k}]_{t-t':k}$, the actions from $A_{t-k}$ that will be played after time $t'$, into the inference call at step $t'$ to improve temporal consistency between chunks. This is achieved via \emph{diffusion inpainting} and results in action predictions that mitigate potentially jerky movements between chunks, improving overall performance. Notably, however, \rtc still only relies on state information available at the start of the inference call, and is not able to update its generation if the environment changes during inference. Furthermore, \rtc utilizes intermediate actions solely to ensure temporal consistency, and is not able to update current predictions based on these actions.

\textbf{Diffusion Steering via Reinforcement Learning (\dsrl).}
We will consider \dsrl \cite{wagenmaker2025steering} as our base RL finetuning algorithm.
\dsrl is an approach for RL improvement of diffusion or flow policies and operates by tuning the \emph{denoising process} of $\pipt$. As noted, if $\pipt$ is a diffusion or flow policy, in standard deployment it generates actions by first sampling ``noise'' vector $w \sim \cN(0,I)$, and then \emph{denoising} $w$ to a robot action. \dsrl trains a lightweight policy, $\pirl$, that instead selects the input noise to be denoised by $\pipt$, allowing the actions produced by $\pipt$ to be ``steered'' to desired behaviors. Formally, if at step $t$ we wish to generate $A_t = \pipt(s_t, w_t)$, instead of $w_t \sim \cN(0,I)$ we would sample $w_t \sim \pirl(s_t)$, then pass $w_t$ to $\pipt$ to be denoised. By observing the response and updating $\pirl$ accordingly, we can improve action generation, enabling more effective policy performance. As we will see in the following, \dsrl enables significantly faster improvement in our setting than other RL approaches such as residual RL \citep{johannink2018residual,ankile2025imitation, yuan2024policy,ankile2025residual}.\loose

\section{Efficient RL Finetuning Under Inference Delays}\label{sec:method}
\begin{figure}[t]
    \centering
    \includegraphics[width=0.9\linewidth]{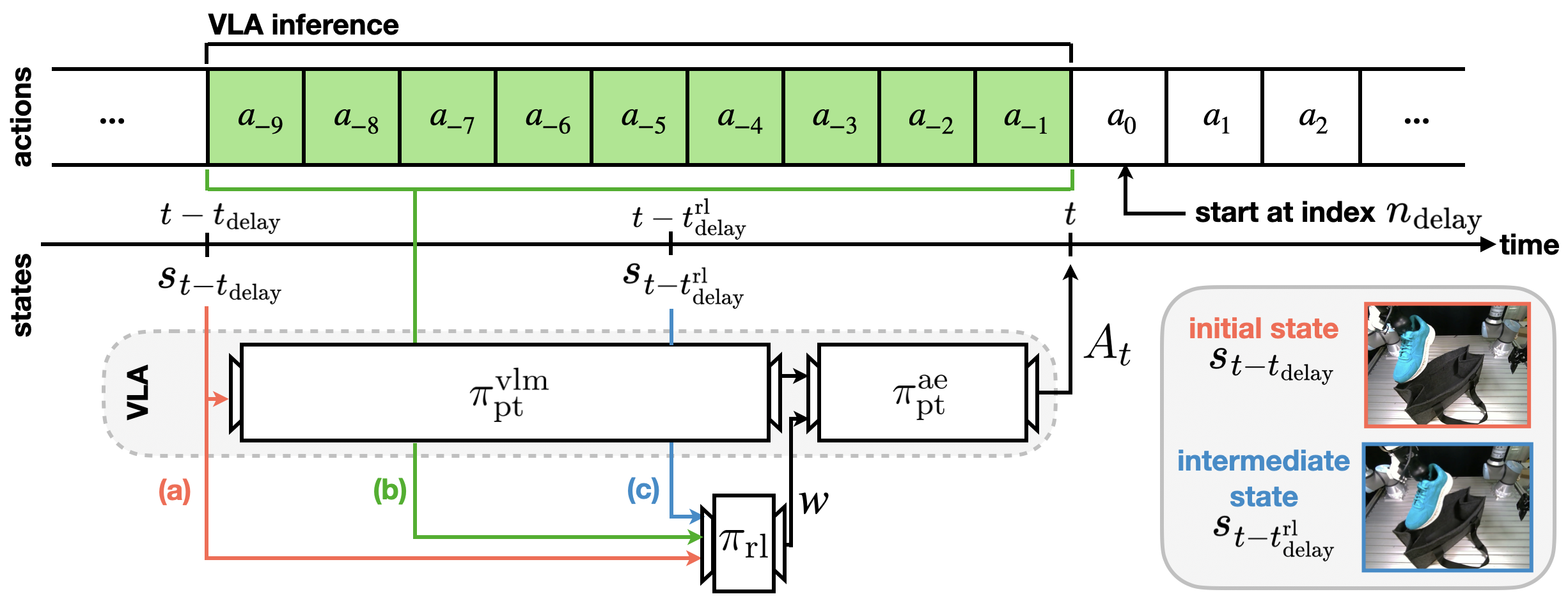}
    \caption{An overview of our approach, \method. Inference begins $\tdelay$ steps before the next action chunk is required, in which the VLM backbone $\piptvlm$ takes state $s_{t-\tdelay}$ as input. At step $t-\tdelayrl$,  RL policy $\pirl$ must be run to ensure that the action expert $\piptae$ is provided a noise input $w$ once VLM inference finishes. $\pirl$ is conditioned on (a) the initial state $s_{t-\tdelay}$, (b) intermediate actions $a_{-9:0}$, and (c) intermediate state $s_{t-\tdelayrl}$. Once the action expert's inference is complete, the next action played, $a_0$, starts at step $n_\text{delay}$ within the newly-generated chunk $A_t$.}
    \vspace{-0.5cm}
    \label{fig:method_fig}
\end{figure}

\iftoggle{arxiv}{
As we will see, in our real-world settings of interest, synchronous inference---where the policy pauses between action chunks to compute the next chunk---introduces problematic delays, and asynchronous inference, as outlined in \Cref{sec:prelim}, is required to enable effective performance. 
We therefore focus on the asynchronous inference setting, and how we can enable efficient online improvement under asynchronous inference. In the following, we outline our proposed approach, asynchronous RL with intermediate information (\method); see \Cref{fig:method_fig} for an overview.\loose
}{
As a starting point, we therefore focus on the asynchronous inference setting outlined in \Cref{sec:prelim}, as the delays introduced by fully synchronous inference are often problematic for  applications of interest. In the following, we outline our proposed approach, \method; see \Cref{fig:method_fig} for an overview.\loose
}

\subsection{Enabling \dsrl Finetuning with Inference Latency}
To enable efficient finetuning of $\pipt$, we consider \dsrl as our core RL algorithm. Assume that we begin the inference of $\pipt$ at step $t - \tdelay$ and condition it on state $s_{t-\tdelay}$.
Naive application of \dsrl to finetune $\pipt$ would condition the noise policy $\pirl$ on the state that $\pipt$ is conditioned on---in this case $s_{t-\tdelay}$---and generate an initial noise to steer the action generation of $\pipt$ based on this state. 
In the asynchronous inference setting, however, the computed action is not played until step $t$---there is a delay of $\tdelay$ steps between when the RL policy computes an action and this action is played---and,
as we will see in \Cref{sec:experiments}, naive application of \dsrl in the asynchronous inference setting therefore does not lead to effective finetuning.
To enable effective RL improvement under asynchronous inference, we consider three key modifications, outlined below.\loose

\textbf{\method actions: Conditioning $\pirl$ on intermediate actions.}
Even though the most current state we have access to when inference begins is $s_{t-\tdelay}$, we also have access to the \emph{intermediate actions}, $a_{t-\tdelay:t} := (a_{t-\tdelay}, a_{t-\tdelay+1},\ldots , a_{t-1})$---the actions from the previously computed action chunk that will be executed during inference. While not fully predictive of the state at time $t$ in general, knowing $s_{t-\tdelay}$ and $a_{t-\tdelay}, a_{t-\tdelay+1},\ldots , a_{t-1}$ gives significant information about where we will be at time $t$, as we can then estimate how the environment evolves from $s_{t-\tdelay}$.

Our first modification is to condition $\pirl$ not just on $s_{t-\tdelay}$, but also the intermediate actions $a_{t-\tdelay:t}$ (see part (a) of \Cref{fig:method_fig}). That is, denoting the ``RL state''---the state we condition the RL policy on---as $\srl_t$, we set $\srl_{t} \leftarrow (s_{t-\tdelay}, a_{t-\tdelay:t})$. Note that while it is not straightforward to condition $\pipt$ on intermediate actions (since $\pipt$ is a frozen policy that cannot easily incorporate additional conditioning information), $\pirl$ can learn to utilize such information as needed, adapting its behavior as it learns how $a_{t-\tdelay:t}$ affects the future state. %

\begin{wrapfigure}[14]{r}{0.56\textwidth}
\centering
\vspace{-0.5em}
\includegraphics[width=1.0\linewidth]{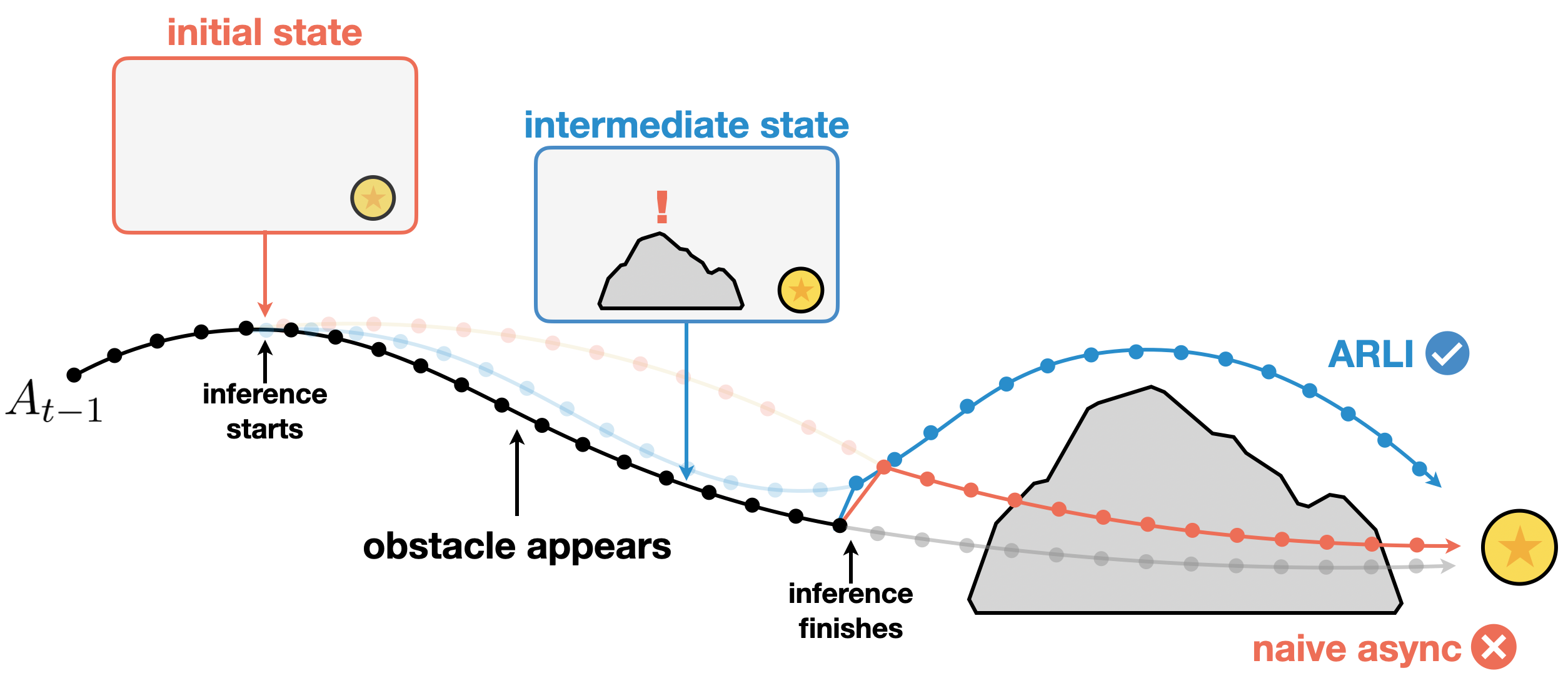} 
\caption{The benefit of intermediate state conditioning: In naive asynchronous inference, the initial state does not capture stochastic events (e.g., moving obstacles).  
By having the ability to condition on an intermediate state, our policy is capable of reacting to disturbances.}
\label{fig:int_state_figure}
\end{wrapfigure}
\textbf{\method state: Conditioning $\pirl$ on intermediate state by decomposing inference delays.} %
While incorporating the intermediate actions into the RL state enables more effective action selection to be played at a future state, it does not account for potential disturbances that may occur between $t-\tdelay$ and $t$; see Figure~\ref{fig:int_state_figure} for an example of this scenario.

To allow for reactivity to such perturbations, we note that most VLAs are decomposed into two components: the \emph{VLM backbone} ($\piptvlm$) which takes as input images and language and outputs a sequence of embeddings, and the \emph{action expert} ($\piptae$) which instantiates the denoising process, taking as input the embeddings produced by $\piptvlm$ and producing the final action. In many cases, $\piptvlm$ is significantly larger and takes longer to run inference for than $\piptae$; the VLA $\pi_0$ spends approximately 2/3 of each inference call on inference of $\piptvlm$, and only the last 1/3 on $\piptae$ inference \cite{black2024pi_0}.

Since \dsrl modifies the behavior of $\pipt$ by selecting the input noise for the denoising process (which is only initialized once $\piptae$ is called), we can begin inference on $\pirl$ right before $\piptae$ is called to gain access to a more recent state. If we denote the combined inference time of $\pirl$ and $\piptae$ as $\tdelayrl$, the latest we can call $\pirl$ is at timestep $t - \tdelayrl$ without incurring additional inference delay. Therefore, we propose further conditioning $\pirl$ on the state at this timestep (see part (b) of \Cref{fig:method_fig}), transforming our RL state to be $\srl_t \leftarrow (s_{t-\tdelay}, a_{t-\tdelay}, s_{t-\tdelayrl})$. Although this intermediate state is still slightly delayed, the dominant latency comes from the VLM backbone, making $\srl_t$ substantially more predictive of $s_t$ compared to using only the initial state and intermediate actions.

\textbf{RL with \rtc.}
Although asynchronous inference does not require \rtc, prior work has shown that \rtc improves zero-shot policy deployment by encouraging temporal consistency between predicted action chunks and reducing discontinuities across inference calls. While such behavior could in principle be learned by $\pirl$, explicitly enforcing it through \rtc ensures this behavior is reliably executed without affecting the learning speed.
Specifically, \rtc freezes the noise for the first $\tdelay$ denoising steps and applies an inpainting guidance term to maintain continuity in later steps. In our approach, we select the initial noise using \dsrl while still incorporating \rtc's inpainting guidance during denoising. We emphasize that, while \rtc incorporates intermediate actions in the VLA's inference, it is \textit{only} to preserve continuity between action chunks; in contrast, $\pirl$ can leverage them to learn substantially richer behaviors.

\newcommand{\Pidelay}{\Pi_{\mathrm{delay}}}
\newcommand{\cD}{\mathcal{D}}
\newcommand{\cO}{\mathcal{O}}

\subsection{RL Under Inference Delays Provably Succeeds}
While \arli seeks to condition the RL policy on as up-to-date state information as possible, we must still plan based on a state delayed by $\tdelay$ steps. We next show that, under certain environment conditions, this will not result in significant policy suboptimality:

\begin{definitionmain}[Delayed Oracle Optimality Gap]
\label{def:delay-oracle}
An MDP $\mathcal{M}$ exhibits $\omega_d$-delayed optimality gap if for any $s_t, a_{t:t+k}$, 
\begin{align}
\Big |\underbrace{Q^\star_{\mathrm{ac}}(s_t, a_{t:t+k})}_{(a)} - \underbrace{\mathbb{E}_{P(\cdot \mid s_{t}, a_{t:t+k})}[R_{t:t+k-d} + \gamma^{k-d} V^\star_{\mathrm{ac}}(s_{t+k-d}, a_{t+k-d:t+k})]}_{(b)} \Big | \leq \omega_{d},
\end{align}
where $Q^\star_{\mathrm{ac}}$ denotes the optimal $Q$-function for action chunk $k$ policies, and we define $V^\star_{\mathrm{ac}}(s_t, a_{t:t+d}) = {\max_{a_{t+d:t+k}}} Q^\star_{\mathrm{ac}}(s_t, a_{t:t+k})$.
\end{definitionmain}
This quantifies the difference in performance between the optimal achievable performance playing $k$-step action chunks $A_t$ given the current observation $s_t$ (term (a)), and the performance if we fully commit to the same action chunk as (a), but then make a decision for the next action chunk $d$-steps early, based on the delayed observation at that step (term (b)).
The following result shows that we can apply standard $Q$-learning approaches to learn effective policies despite our $\tdelay$ observation delay.\loose
\newcommand{\pihat}{\widehat{\pi}}
\begin{propositionmain}[Informal]\label{prop:delay_subopt}
Consider applying standard $Q$-learning  to observations delayed by $\tdelay$ steps and action chunks of length $k$. Then the policy learned through this, $\pihat^{\tdelay}$, satisfies:
$V^{\star}_{\mathrm{ac}} - V(\pihat^{\tdelay}) \le  \frac{1}{1-\gamma^{k-\tdelay}} \cdot \omega_{\tdelay} $ for $V^{\star}_{\mathrm{ac}}$ the optimal value of a policy that places action chunks of length $k$ with no delay, and
$\omega_{\tdelay}$ the oracle optimality gap of $\cM$ under delay $d \leftarrow \tdelay$.
\end{propositionmain}
Proposition \ref{prop:delay_subopt} shows that, as long as observation delays do not significantly impact the performance of the \emph{oracle} action chunking policy, we can apply standard $Q$-learning approaches (in particular, the action chunking $Q$-learning approach of \cite{li2025decoupled}) to learn a policy that will converge to approximately the loss in performance the oracle policy incurs. 
As \method adopts a $Q$-learning approach analogous to \cite{li2025decoupled}, this shows that, while \method may incur some suboptimality from delayed observations, this suboptimality can be effectively bounded. Furthermore, by mitigating the delay with intermediate states, \method can reduce the delayed oracle optimality gap, leading to more effective performance than can be achieved without this.
We present the full version of Proposition \ref{prop:delay_subopt} in \Cref{appdix:theory}.

%% file: body/experiments.tex
\newcommand{\alohacube}{\texttt{AlohaTransferCube}\xspace}

\section{Experimental Results}\label{sec:experiments}
We next test whether \arli enables effective improvement of generalist policies in practice. We aim to evaluate whether standard RL approaches enable effective improvement under asynchronous VLA inference and whether \arli is able to improve on these approaches, leading to both faster learning and higher final success rate. We evaluate \arli on several simulated tasks that require high reactivity, as well as on three real-world tasks.
\begin{figure}
    \centering
    \includegraphics[width=0.97\textwidth]{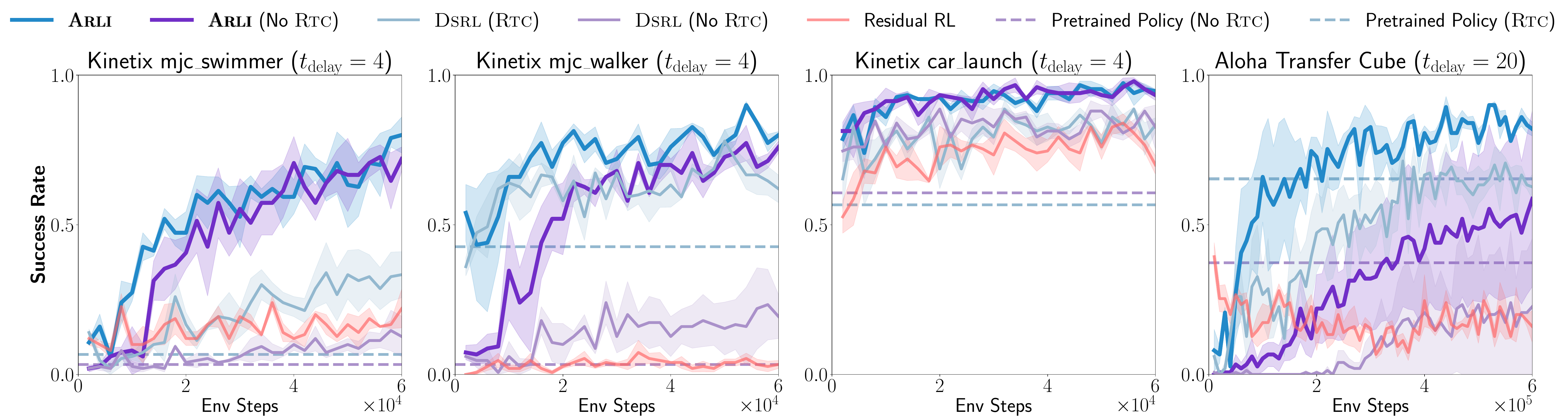}
    \caption{Comparison of \arli and \dsrl on simulation tasks with and without \rtc.}
    \label{fig:sim_baseline_results}
    \vspace{-1em}
\end{figure}

\newcommand{\aresid}{a^{\Delta}}
\textbf{Baseline methods.}
We compare \arli to several other approaches that integrate RL into asynchronous VLA inference. First, we consider the most straightforward application of asynchronous \dsrl, where we condition the RL policy on $s_{t-\tdelay}$. We then consider variants of \dsrl with all possible state augmentations introduced in \Cref{sec:method}, including intermediate actions, intermediate states, \rtc, and different combinations thereof. For our main simulation results, we consider the direct application of \dsrl with and without \rtc 
and \arli with and without \rtc; we leave other comparisons for the ablations (\Cref{sec:ablations}). For real-world experiments, we compare against a synchronous inference approach which simply pauses to compute each new action.
We run \arli as described in \Cref{sec:method}.

In addition to variants of \dsrl, we also compare \arli to residual RL \citep{johannink2018residual,ankile2025imitation, yuan2024policy,ankile2025residual}. Residual RL operates by training an RL policy $\pirl$ to output a \emph{residual correction} $\aresid_t$, and then executing the action $a_t + \aresid_t$, for $a_t$ the action produced by the base policy. Unlike \dsrl which must be incorporated into the inference process of the base policy, residual RL can be run independently of this inference. As $\pirl$ is often parameterized as a small MLP that has minimal inference cost, we can compute a fresh residual correction at each step based on the current state, effectively mitigating the effect of the inference delay. For our experiments, we utilize the instantiation of residual RL proposed in \cite{ankile2025residual}, which conditions $\pirl$ on $s_t$ and $a_t$, and computes a residual correction at each step. 

 For all simulation results, we average results across 3 seeds and report error bars denoting 1 standard error. Please see the Appendix for additional experimental details.

\subsection{Simulated Experiments}\label{sec:sim_experiments}
We first evaluate \arli on several simulated tasks.
We consider the \textbf{Kinetix} benchmark \cite{black2026real}, which contains dynamic environments, requiring high reactivity. For our pretrained policy $\pipt$ we utilize the publicly available flow policy checkpoints considered in \cite{black2026real}.
For our comparison, we select the \texttt{mjc\_swimmer}, \texttt{mjc\_walker}, and \texttt{car\_launch} tasks, which have lower pretrained policy success rates, and show significant degradation in zero-shot performance over increasing inference latencies. 
For each task, we set $\tdelay = 4$ and $\tdelayrl = 1$. 
We also consider the \textbf{AlohaTransferCube} task \cite{zhao2023learningfinegrainedbimanualmanipulation}, which simulates VLA inference with larger latencies and tests \arli's ability to scale to more challenging bimanual manipulation tasks. For $\pipt$ we utilize the publicly available Aloha finetune of $\pi_0$ \cite{black2024pi_0}, a 3.3B-parameter VLA. We set $\tdelay = 20$ and $\tdelayrl = 10$.

Figure \ref{fig:sim_baseline_results} plots policy success rate over time for the 4 simulated tasks, and shows that \arli significantly outperforms naive application of \dsrl under asynchronous inference, converging to both a higher final success rate and requiring less training time. While running \dsrl with \rtc does improve the finetuning performance on some tasks, we find that this is insufficient to fully learn the desired behavior, and that the state augmentations incorporated in \arli are required. Furthermore, we see that, while in many cases \arli is able to perform effectively without \rtc, \rtc does improve the efficiency and reliability of \arli. We see as well that \arli enables significant improvements over residual RL---despite the higher reactivity residual RL enables, incorporating \dsrl's ability to steer the denoising process enables significantly higher final success once we incorporate \arli's state augmentations.
These results illustrate that \arli is able to effectively handle the asynchronous inference delays, and still enable highly effective finetuning, both in high reactivity settings, as well as with large-scale VLAs.

\subsection{Real-World Experiments}

\begin{figure}
    \centering
    \includegraphics[width=1\linewidth]{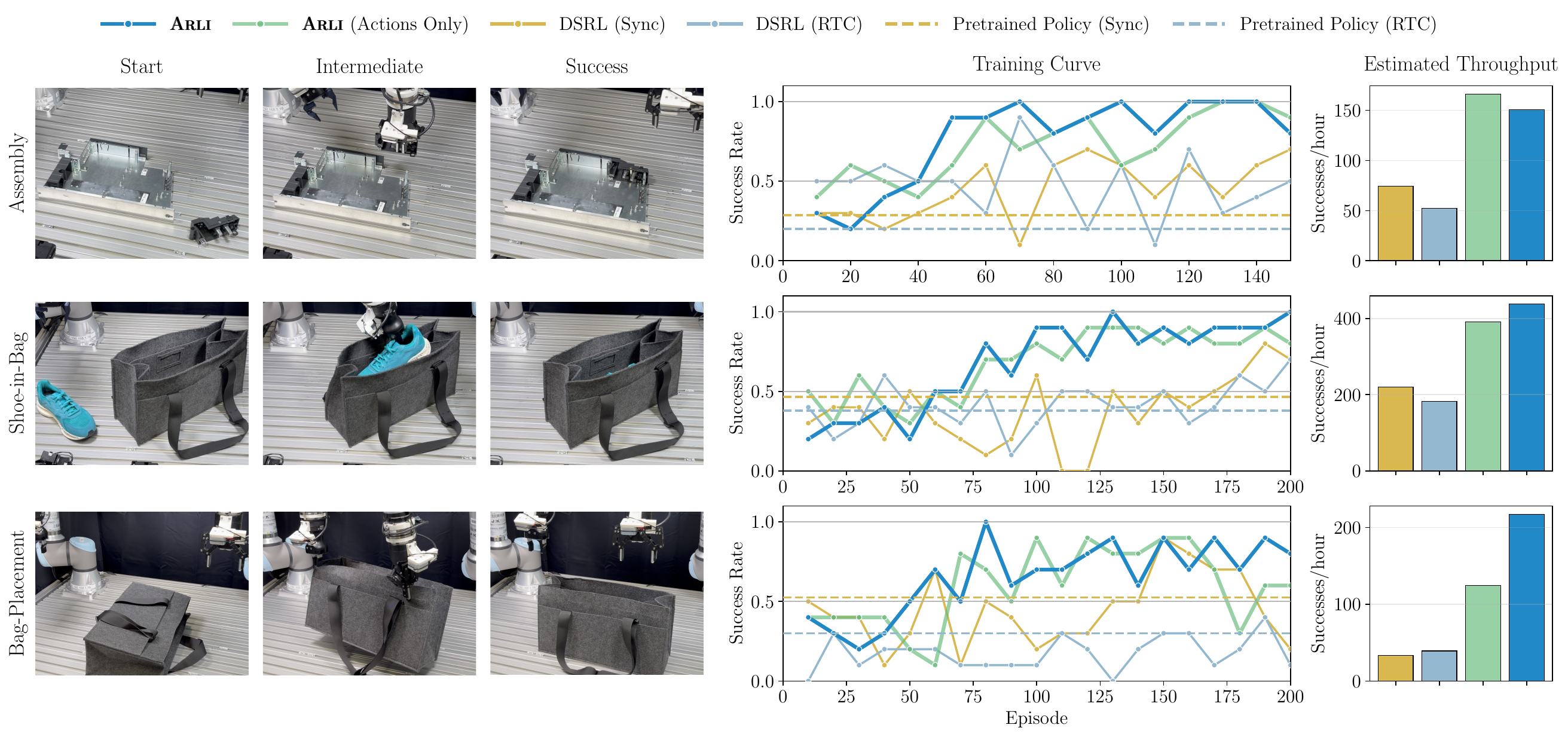}
    \vspace{-1.5em}
    \caption{
    Comparison of \arli against other RL finetuning methods on real-world tasks. \textit{Left}: task images. \textit{Middle}: training success curves corresponding to average success rate in the last 10 episodes. \textit{Right}: estimated throughput in successes/hour computed based on the duration of each success in the last 20 episodes
    and assuming failure is declared after 2x average success duration. 
    }
    \label{fig:Opt1_real_world_results}
    \vspace{-1em}
\end{figure}

We next test whether \arli allows for effective RL finetuning in real-world robotic deployment. We consider the following real-world tasks (see \Cref{fig:Opt1_real_world_results} for visualization):
\begin{itemize}[nosep,leftmargin=*]
    \item \textit{Assembly}. Insert a power connector into an inverter box by aligning it with two pins.
    \item \textit{Shoe-in-Bag}. Pick up a shoe and place it into a narrow bag. %
    \item \textit{Bag-Placement}. With both arms, pick up a lying bag and position it parallel to the front. 
\end{itemize}
 \suppdetail{For all tasks, we start each episode from the same initial condition to make the results comparable (with the exception of the handles of the bag during \textit{Bag-Placement}).} 
 We run all tasks on a bimanual UR5e robot cell. Each robot is equipped with a wrist camera and a base camera is installed at the front of the cell. For all experiments, we initialize the base policy with $\pi_{0.5}$ \cite{intelligence2025pi_} finetuned on a small dataset of human demonstrations on the target task at 60hz and action chunk length 50. On an NVIDIA GeForce RTX 5090 GPU, we find $\tdelay=10$, $\tdelayrl=7$, and action horizon $n=20$ to be optimal for both sync and RTC policies. Note that with RTC, inference time needed for $\piptae$ increases, and we find that $\piptvlm$ inference takes up only ~40\% of the total inference time.

Figure~\ref{fig:Opt1_real_world_results} shows the results of training on these three tasks. Starting from a base policy success rate of around 40\%, \arli is capable of achieving near 100\% success on all tasks after 100 to 125 episodes of training. This is a significant improvement over \dsrl (synchronous) and \dsrl (\rtc), which struggle to reach 80\% in the same amount of time on the \textit{Assembly} and \textit{Bag-Placement} tasks, and take nearly double the time to match \arli in the \textit{Shoe-in-Bag} task. Interestingly, \arli (Actions Only) performs comparably to \arli with the exception of the \textit{Bag-Placement} task, where it fails to converge to a high success rate.
We highlight that without asynchronous inference, both the base policy exhibits poor success and RL training is significantly slower, or does not learn at all. In other words, asynchronous inference is required to learn on these tasks, and \arli is capable of facilitating RL finetuning with asynchronous inference, allowing for online VLA improvement in settings where this would previously have not been possible.
We also estimate policy throughput in successes/hour and find that \arli shows significantly higher throughput across all tasks.

\subsection{Ablations}\label{sec:ablations}

\begin{figure*}
    \centering
    \begin{minipage}[t]{0.49\textwidth}
        \centering

        \includegraphics[width=0.95\textwidth,height=1.45in]{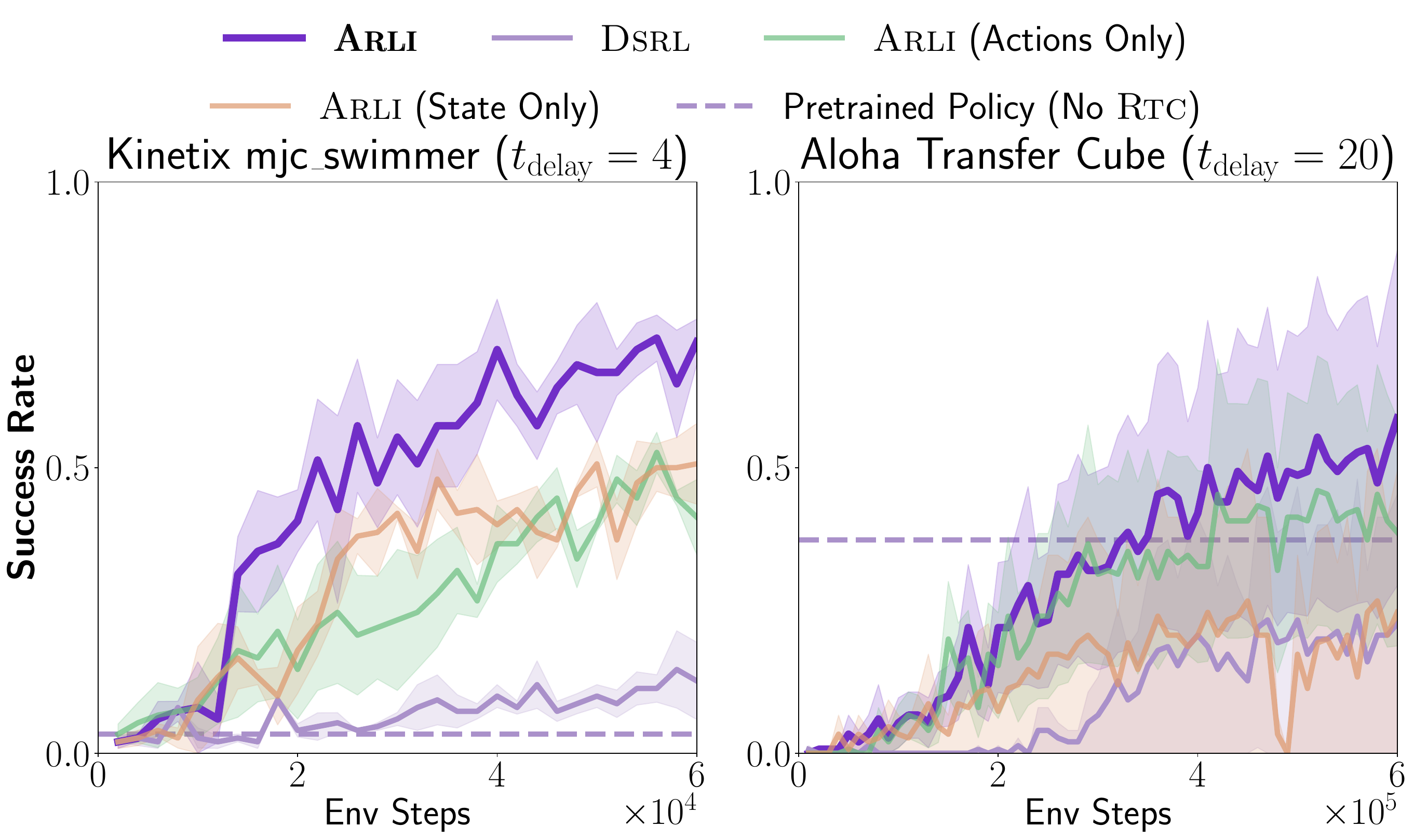}
        \caption{Variants of \arli (no \rtc) with one of intermediate state or actions on \texttt{mjc\_swimmer} and \alohacube.}
        \label{fig:intermediates_ablation}
    \end{minipage}%
    ~ 
    \begin{minipage}[t]{0.49\textwidth}
        \centering
        \includegraphics[width=0.95\textwidth]{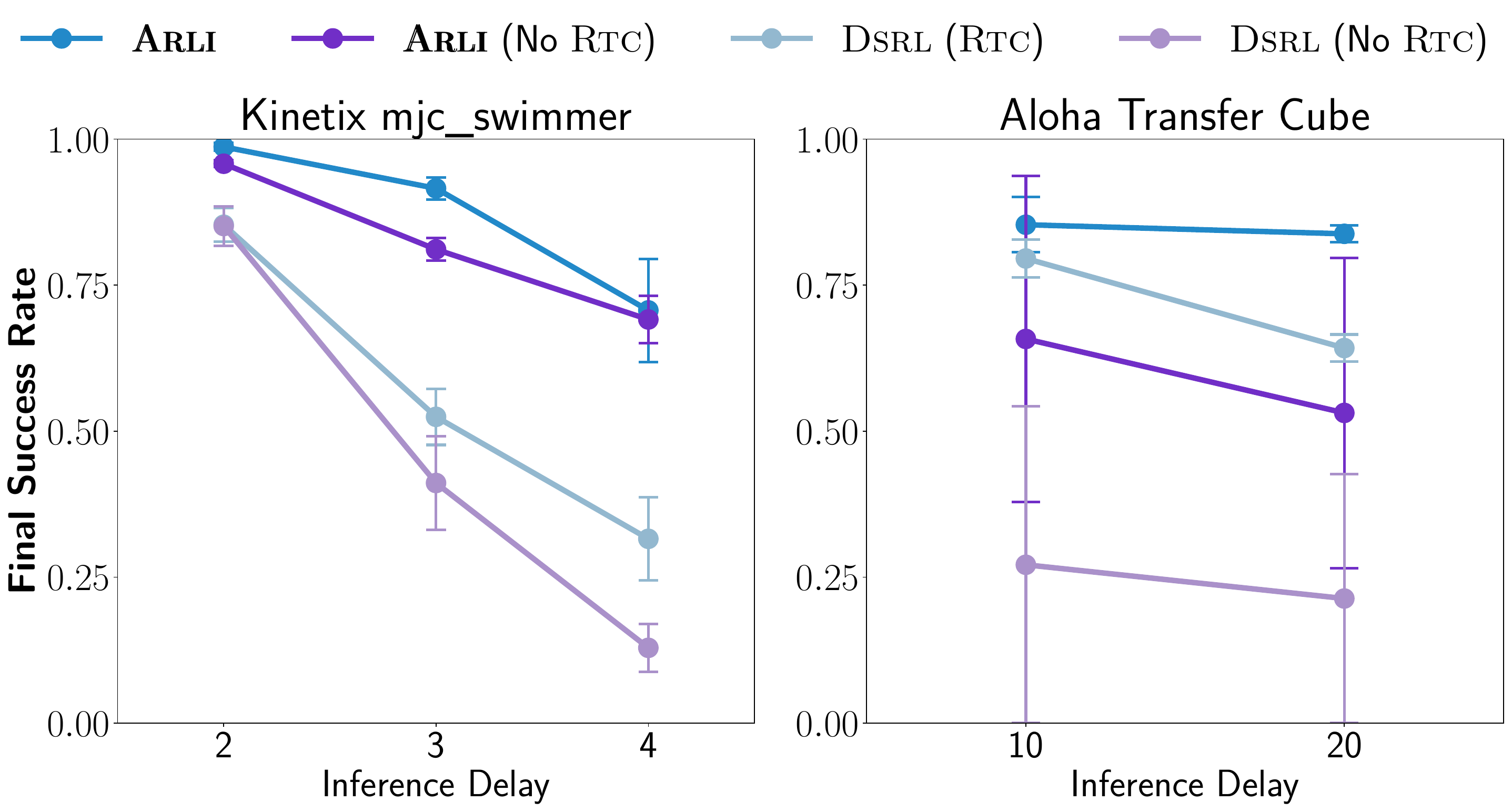}
        \caption{Comparison of \arli and \dsrl across inference delays on \texttt{mjc\_swimmer} and \alohacube.}
        \label{fig:delay_ablation}
    \end{minipage}
    \vspace{-1em}
\end{figure*}

We next investigate the efficacy of our design choices through a series of ablation studies.

\textbf{How critical is it that we condition $\pirl$ on the intermediate state and actions?} We evaluate the importance of including intermediate state and actions by running the experiment of Section \ref{sec:sim_experiments} on the \texttt{mjc\_swimmer} and \alohacube tasks, isolating the usage of the intermediate state and actions. We run these experiments \textit{without} \rtc to minimize the temporal information utilized by default. Figure \ref{fig:intermediates_ablation} shows that conditioning $\pirl$ on both intermediate states and actions yields significantly better results than conditioning on one of the two, highlighting the importance of our choice of $\srl_t$.\loose

\textbf{How sensitive is \arli to $\tdelay$?} We evaluate the sensitivity of \arli to $\tdelay$ by running the same experiment in Section \ref{sec:sim_experiments} for the \texttt{mjc\_swimmer} and \alohacube tasks, setting $\tdelay=[2,3,4]$ for Kinetix and $\tdelay=[10,20]$ for \alohacube. We accordingly adjust the length of the intermediate actions to $\tdelay$ and set $\tdelayrl=\tdelay-1$ for Kinetix and $\tdelayrl=\tdelay/2$ for \alohacube.
See Figure \ref{fig:delay_ablation} for plots of final success rate of \arli against the baselines. Across both tasks, we see that with \rtc the performance of \arli decays slower than \dsrl as $\tdelay$ increases, indicating that \arli is more robust to inference delay.

\begin{wrapfigure}[10]{r}{0.4\textwidth}
\centering
\vspace{-1.5em}
\includegraphics[width=0.9\linewidth]{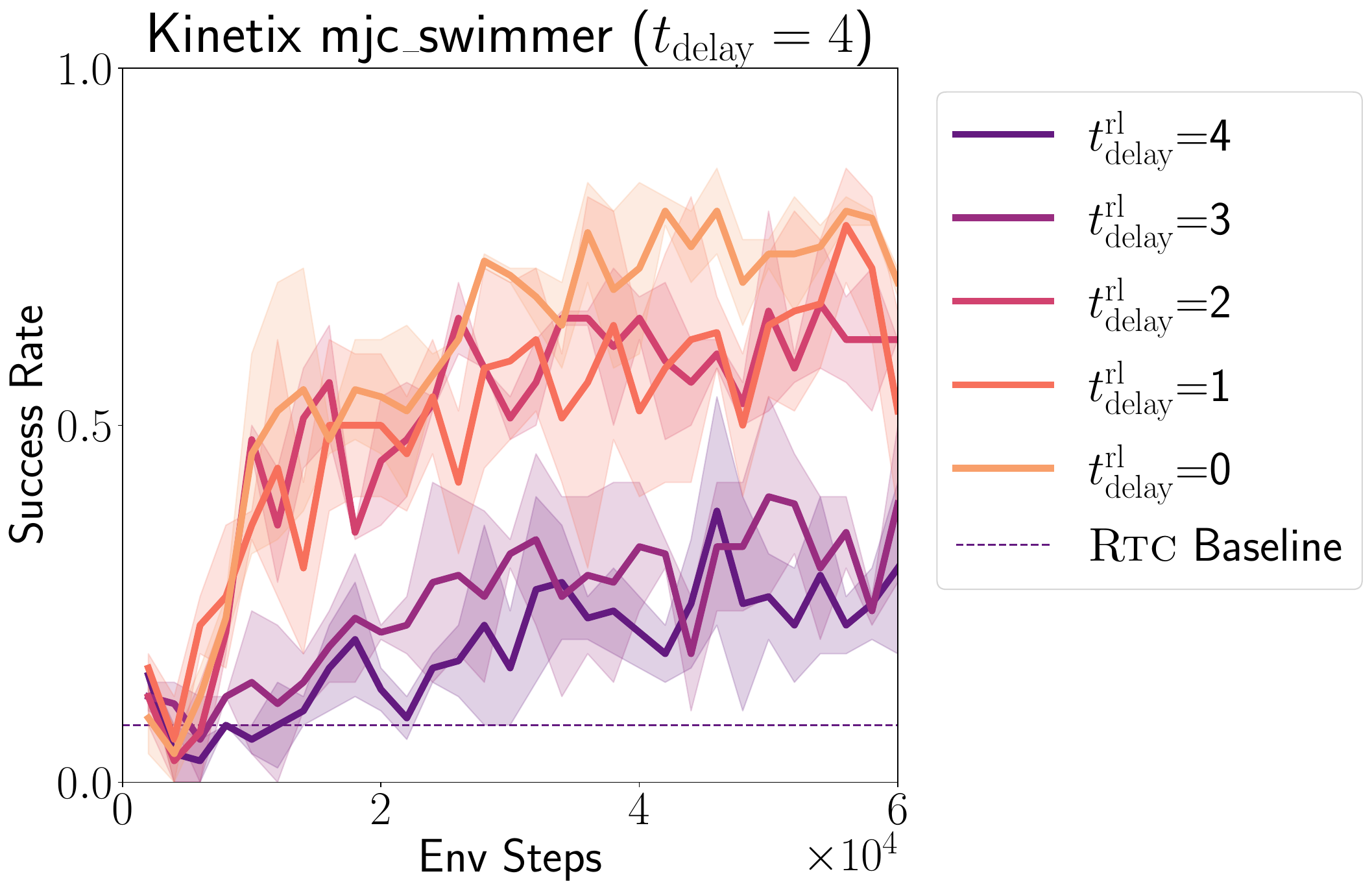} 
\caption{\arli performance at varying RL inference delays on \texttt{mjc\_swimmer}.}
\label{fig:mid_obs_figure}
\end{wrapfigure}

\textbf{How sensitive is \arli to $\tdelayrl$?} We evaluate the sensitivity of \arli against RL inference latency by running the same experiment in Section \ref{sec:sim_experiments} for the \texttt{mjc\_swimmer} task, setting $\tdelay=4$ and ablating $\tdelayrl$ from 0 (most up-to-date) to 4 (no additional information). For these experiments we run \arli with \rtc. See Figure \ref{fig:mid_obs_figure} for results, which show that having more up-to-date information is better, and that when $\tdelayrl$ is greater than half of $\tdelay$, learning significantly improves.

%% file: body/limitations.tex
\section{Conclusion and Limitations}
\vspace{-0.2cm}

In this work, we have proposed \arli, an approach to enable effective RL improvement of generalist robot policies under practical inference constraints. In particular, we have shown that, by properly augmenting the observations the RL policy is conditioned on, we can enable effective RL improvement with asynchronous policy inference, enabling learning while mitigating the effects of latency.\loose

While \arli has shown effective performance on our simulated and real-world tasks, it does possess a number of limitations. A key assumption we make for our pretrained policy is that it can be split up into VLM backbone and action expert components, where the action expert is the only part of the policy that requires a noise input. While this is true for many popular VLAs, for policies that do not satisfy this constraint, we cannot utilize intermediate states to improve performance (though \arli can still take advantage of intermediate actions in this case, as our experimental results demonstrate). 
While introducing \rtc to our method shows performance gains, utilizing \rtc might in fact hurt the reactivity of our policy, as it forces old actions to be played, making the RL policy less expressive. 
Lastly, while \dsrl is capable of learning in many tasks, previous works have observed that \dsrl can struggle to learn on certain tasks. Given \arli's reliance on \dsrl, \arli will likely also struggle to learn on such tasks.

%% file: body/appendix.tex
\newcommand{\dsrlsac}{\textsc{Dsrl-Sac}\xspace}
\newcommand{\single}{\mathrm{single}}
\newcommand{\sac}{\textsc{Sac}\xspace}

\section*{Contributions}

\begin{itemize}[leftmargin=*]
\item Brian Zhu: Methodology (contribution 2: adding intermediate state), DSRL + RTC repro, Aloha simulation, software architecture \& implementation, writing
\item Momen Khalil: Methodology (contribution 1: adding intermediate actions), DSRL + RTC repro, Kinetix simulation, software architecture \& implementation, writing 
\item E Harrison: Simulated experiments, writing
\item Emanuele Poggi: Simulation and real-world experiments, ran learning dynamics analysis, writing 
\item Philipp Schmitt: original idea, software architecture, advice
\item Bernd Kast: Software architecture \& implementation, advice
\item Philine Meister: Real-world experiment lead, writing
\item Pranav Atreya: Experimental support, writing
\item Qiyang Li: Theoretical results
\item Finn Ferchau: Real-world experiment execution
\item Cesar Colmenero: Real-world experiment execution
\item Yash Shahapurkar: Early real-world experiment testing
\item Gokul Narayanan: Early real-world experiment testing
\item Melih Erdogan: Early real-world experiment testing
\item Kai Wurm: Advice, resources, funding
\item Georg von Wichert: Advice, resources, funding
\item Oier Mees: Advising, writing
\item Eugen Solowjow: Advice, resources, funding
\item Andrew Wagenmaker: Advising, method ideation, writing
\item Sergey Levine: Advising, writing
\end{itemize}

\section{Experimental Details}

\subsection{Implementation}

\textbf{Algorithm}. For all experiments, we run \arli and \dsrl using \dsrlsac with repeated latent actions as described in the $\pi_0$ experiments from \cite{wagenmaker2025steering}. In particular, we train a single step actor $\pirl^\single$ and critic $Q^\single(s,w_\single)$ where $\mathcal{W}_\single = \mathbb{R}^d$; during inference we sample $w_\single \sim \pirl^\single(s)$ and repeat it across the action chunk axis to query the flow policy, i.e. $w = \{(w_1, \ldots, w_k) | w_i = w_\single\}$ where $k$ is the chunk length of the flow policy.
Similar to \cite{wagenmaker2025steering}, we handle RL with action chunking by treating the sampling and execution of an action chunk as a single step in the latent MDP, ignoring observations from within the chunk. Note that the number of environment steps stated in all results corresponds to the total number of steps from the original environment rather than the latent MDP.

For residual RL experiments, we use the implementation described in \cite{ankile2025residual}. Specifically, we train residual policy $\pirl$ to output a residual correction $\aresid_t$, which is used to edit the base policy's action and execute $\tilde{a}_t = a_t + \aresid_t$. The resulting action $\tilde{a}_t$ is also used to train the critic function $Q(s_t, \tilde{a}_t)$. The residual RL policy can be run independently of the base policy's inference process, and due to its lightweight size, can essentially be run without inference delay. So, we provide $\pirl$ with the most up-to-date RL state, $\srl_t \leftarrow s_{t}$. Unlike \dsrl which necessarily stays within the support of the base policy, residual edits that have a similar magnitude to the base policy actions can lead to the policy relearning the task from scratch, rather than finetuning the base policy. To ensure that this does not occur, we limit the search of the optimal action magnitude for residual edits ($b_w$) within [0, 0.1].

\textbf{Observation Space and Architecture}. For Kinetix tasks we use the "symbolic" observation space, which according to \cite{matthews2025kinetixinvestigatingtraininggeneral} is a flat vector that concatenates the physical properties all the entities in the scene (e.g., position, rotation, velocity). The actor and critics are thus defined as an MLP that processes this observation vector. For \alohacube and real-world tasks we use the raw images and joint state as the observation space. For both actor and critics, the raw images are processed by a vision encoder consisting of 4 CNN layers before being flattened and concatenated with the joint state. The final state vector is then processed by an MLP. Table \ref{tab:obs_space_and_arch} summarizes the observation space and architecture choices for each task. Note that for real-world tasks, we use 3 camera streams (1 base camera along with 1 wrist camera for each arm); the images are concatenated along the channel dimension before being processed by the vision encoder.

\begin{table*}[h] 
    \caption{Observation space and architecture used for each environment}
    \label{tab:obs_space_and_arch}
    \centering
    \begin{tabular}{l c c c}
    \toprule
    \textbf{Hyperparameter} & \textbf{Kinetix} & \textbf{AlohaTransferCube} & \textbf{Real-World} \\ 
    \midrule
    Observation space & Symbolic & Pixels \& State & Pixels \& State \\
    Image observation size & N/A & 64 $\times$ 64 & 64 $\times$ 64 \\
    Actor and critic architecture & MLP & CNN+MLP & CNN+MLP \\
    Number of actor and critic MLP layers & 3 & 3 & 3\\
    MLP hidden size & 256 & 128 & 128 \\ 
    CNN features & N/A & $(32, 32, 32, 32)$ & $(32, 32, 32, 32)$ \\
    Number of images & N/A & 1 (\texttt{cam\_high}) & 3 (base \& wrists) \\
    \bottomrule
    \end{tabular}
\end{table*}

\textbf{Intermediate Information Processing}. To handle the intermediate state introduced by \arli, we  concatenate the intermediate state with the original state. For Kinetix, we concatenate the observation vectors. For \alohacube and real-world tasks, we concatenate images along the channel dimension and separately concatenate the joint states. To handle the intermediate actions introduced by \arli, we flatten the actions and concatenate them to the full state vector before it is processed by the MLP layers.

\textbf{Code}. We use the \texttt{dsrl\_pi0} codebase from \cite{wagenmaker2025steering} to implement the \dsrlsac training loop. We modify \texttt{dsrl\_pi0} and \texttt{openpi} \cite{openpi} to integrate \rtc based on the codebase from \cite{black2026real}. We also integrate the Kinetix tasks from \cite{black2026real} with \dsrl. For \alohacube, we do not make additional modifications.

\textbf{Base Policy}. For Kinetix, we use the BC flow policy checkpoints from \cite{black2026real} for each task. In particular, we use the checkpoint for the 5th epoch as the base policy initialization. For \alohacube, we follow \cite{wagenmaker2025steering} and use the $\pi_0$ checkpoint from \texttt{s3://openpi-assets/checkpoints/pi0\_aloha\_sim}. For real-world experiments, we finetune $\pi_{0.5}$ \cite{intelligence2025pi_} with LoRA on each task using a cosine decay schedule. Training data was recorded for the specific tasks at 60Hz using a Meta Quest to remote control the robots. Note that the training data for the Shoe-in-Bag task does not only focus on packing the shoe, but also contains demonstrations for packing other clothing items as t-shirts into the bag. Dataset details and training hyperparameters for finetuning $\pi_{0.5}$ are listed in Table \ref{tab:real_world_base_policy}. 

\begin{table*}[h] 
    \caption{Real-world experiments base policy training data details}
    \label{tab:real_world_base_policy}
    \centering
    \begin{tabular}{l c c c}
    \toprule
    \textbf{} & \textbf{Assembly} & \textbf{Shoe-in-Bag} & \textbf{Bag-Placement} \\ 
    \midrule
    Number of episodes & 60 & 123 & 51 \\
    Total dataset steps & 84821 & 474301 & 50041 \\
    Batch size & 32 & 32 & 32 \\
    Approx. number of epochs & 3 & 2 & 4 \\
    Warm-up steps & 265 & 1000 & 1000\\
    Decay steps & 8000 & 30000 & 30000\\
    Training steps & 8000 & 30000 & 6000\\
    Peak-LR & 2.5e-5 & 2.5e-5 & 2.5e-5\\
    Decay-LR & 2.5e-6 & 2.5e-6 & 2.5e-6\\
    \bottomrule
    \end{tabular}
\end{table*}

\subsection{Hyperparameters}

\dsrl-specific hyperparameters are listed in Table \ref{tab:dsrl_hyperparameters}. Simulation evaluation metrics are averaged over 50 rollouts. For all experiments with \rtc, we use the hyperparameters listed in Table \ref{tab:rtc_hyperparameters}. Values of inference delay we ablate over are shown as a list. Residual RL hyperparameters are listed in Table \ref{tab:residual_hyperparameters}.

We initialize all experiments with offline data collection (that is, rolling out the base policy some number of episodes) and optionally offline training (that is, training the actor and critic with \sac on the offline data for some number of steps). Hyperparameters are listed in Table \ref{tab:offline_hyperparameters}. Note that number of offline transitions refer to transitions from the latent MDP rather than the original environment. For \alohacube, we found that scaling down the variance of the noise distribution boosted the number of successful trajectories in the replay buffer. Similarly, for real-world experiments we found that clipping the sampled noise $w_\single$ reduces unsafe motions from the robot.

\begin{table*}[h] 
    \centering
    \caption{Hyperparameters used for \dsrl}
    \label{tab:dsrl_hyperparameters}
    \begin{tabular}{l c c c}
    \toprule
    \textbf{Hyperparameter} & \textbf{Kinetix} & \textbf{AlohaTransferCube} & \textbf{Real-World} \\ 
    \midrule
    Batch size & 256 & 256 & 256 \\ 
    Actor learning rate & 0.0001 & 0.0001 & 0.0001\\
    Critic learning rate & 0.0003 & 0.0003 & 0.0003 \\
    Temperature learning rate & 0.0003 & 0.0003 & 0.0003 \\
    Activation & ReLU & ReLU & ReLU \\
    Discount factor & 0.999 & 0.999 & 0.999 \\
    Target entropy & $-d/2$ & $-d/2$ & $-d/2$ \\
    Target update rate ($\tau$) & 0.005 & 0.005 & 0.005 \\
    Gradient steps per update & 20 & 25 & 20 \\
    Action magnitude ($b_w$) & 1.0 & 2.0 & 0.5\\
    Action chunk size & 8 & 50 & 50 \\
    Number of actions to execute & 4 & 25 & 20\\
    Number of critics & 10 & 10 & 10 \\
    Clipped Q-learning & False & False & False\\
    \bottomrule
    \end{tabular}
\end{table*}

\begin{table*}[h] 
    \caption{Hyperparameters used for \rtc}
    \label{tab:rtc_hyperparameters}
    \centering
    \begin{tabular}{l c c c}
    \toprule
    \textbf{Hyperparameter} & \textbf{Kinetix} & \textbf{AlohaTransferCube} & \textbf{Real-World} \\ 
    \midrule
    Denoising steps & 5 & 10 & 10 \\
    Inference Delay ($\tdelay$) & $[1, 2, 3, 4]$ & $[10, 20]$ & 10 \\
    Soft Masking & True & True & True \\
    Max. Guidance Weight & 5 & 10 & 10\\
    \bottomrule
    \end{tabular}
\end{table*}

\begin{table*}[h] 
    \centering
    \caption{Hyperparameters used for Residual RL}
    \label{tab:residual_hyperparameters}
    \begin{tabular}{l c c c}
    \toprule
    \textbf{Hyperparameter} & \textbf{Kinetix} & \textbf{AlohaTransferCube} \\ 
    \midrule
    Batch size & 256 & 256 \\ 
    Actor learning rate & 0.000001 & 0.000001 \\
    Critic learning rate & 0.0001 & 0.0001 &  \\
    Activation & ReLU & ReLU  \\
    Discount factor & 0.99 & 0.99  \\
    Target update rate ($\tau$) & 0.005 & 0.005  \\
    Gradient steps per update & 4 & 4  \\
    Action magnitude ($b_w$) & 0.1 & 0.1 \\
    Action chunk size & 8 & 50 \\
    Number of actions to execute & 4 & 25 \\
    Number of critics & 10 & 10  \\
    Clipped Q-learning & False & False \\
    \bottomrule
    \end{tabular}
\end{table*}

\begin{table*}[h] 
    \caption{Hyperparameters used for initial offline data collection}
    \label{tab:offline_hyperparameters}
    \centering
    \begin{tabular}{l c c c}
    \toprule
    \textbf{Hyperparameter} & \textbf{Kinetix} & \textbf{AlohaTransferCube} & \textbf{Real-World} \\ 
    \midrule
    Number of offline transitions & 500 & 1000 & 1200 \\
    Offline training steps & 0 & 0 & 24000 \\
    Offline noise distribution & $\mathcal{N}(0, 1)$ & $\mathcal{N}(0, 0.3)$ & $\mathcal{N}(0, 1)$ \\
    Clip noise & False & False & True (0.5) \\
    \bottomrule
    \end{tabular}
\end{table*}

\clearpage

\section{Additional Ablations}

\begin{figure*}
    \centering
    \includegraphics[width=0.97\textwidth]{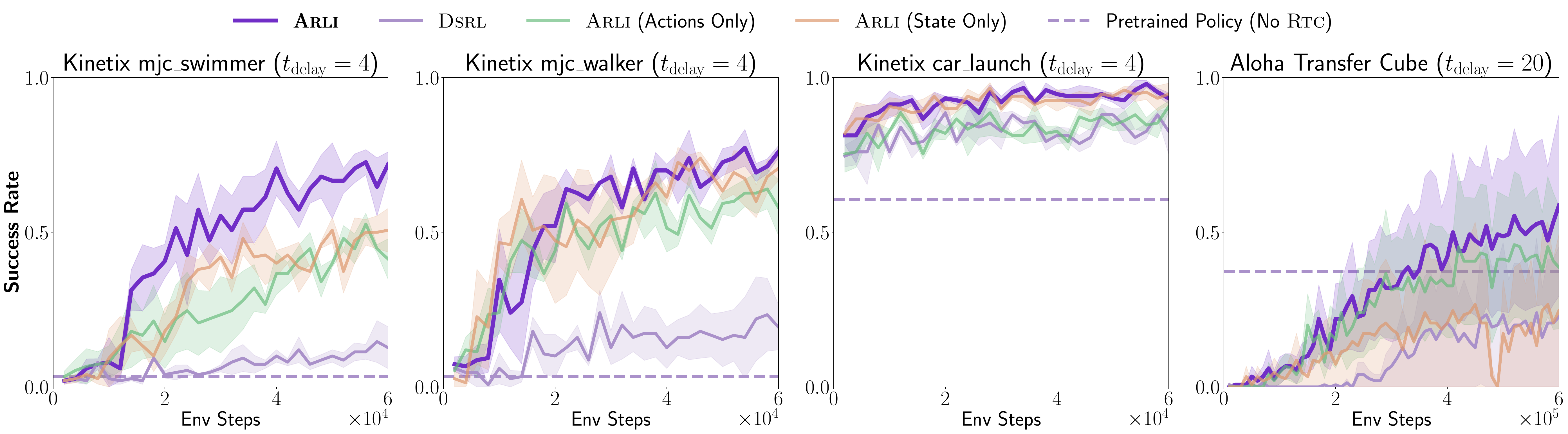}
    \caption{Variants of \arli (no \rtc) with one of intermediate state or actions on \texttt{mjc\_swimmer}, \texttt{mjc\_walker}, \texttt{car\_launch} and \alohacube.}
    \label{fig:state_repr_naive_all}
\end{figure*}

\begin{figure*}
    \centering

    \includegraphics[width=0.97\textwidth]{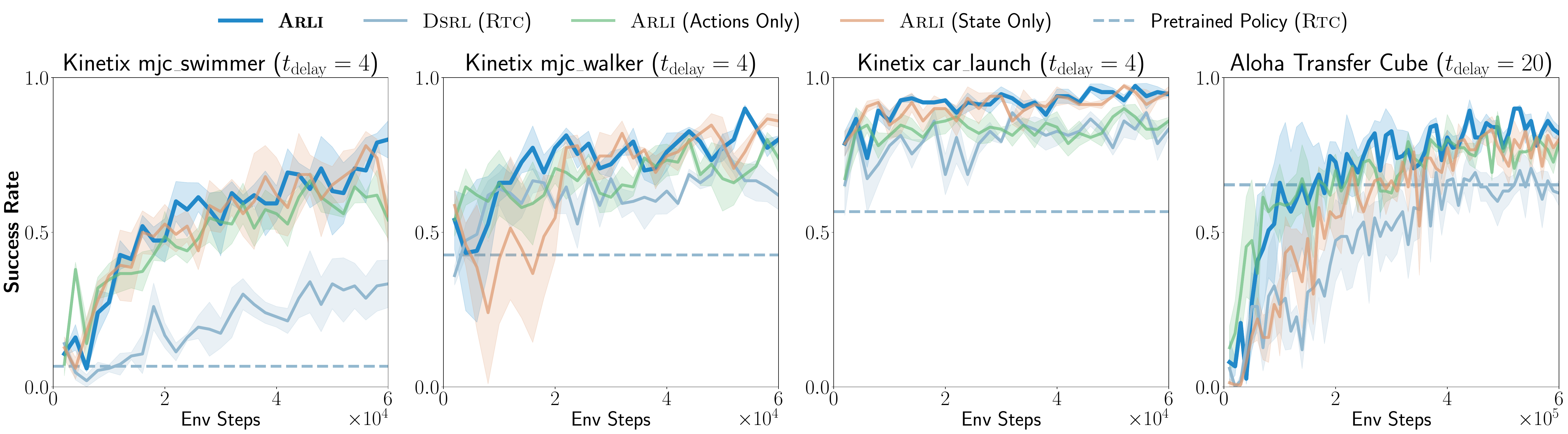}
    \caption{Variants of \arli (with \rtc) with one of intermediate state or actions on \texttt{mjc\_swimmer}, \texttt{mjc\_walker}, \texttt{car\_launch} and \alohacube.}
    \label{fig:state_repr_rtc_all}
\end{figure*}

\begin{figure*}
    \centering
    \includegraphics[width=0.97\textwidth]{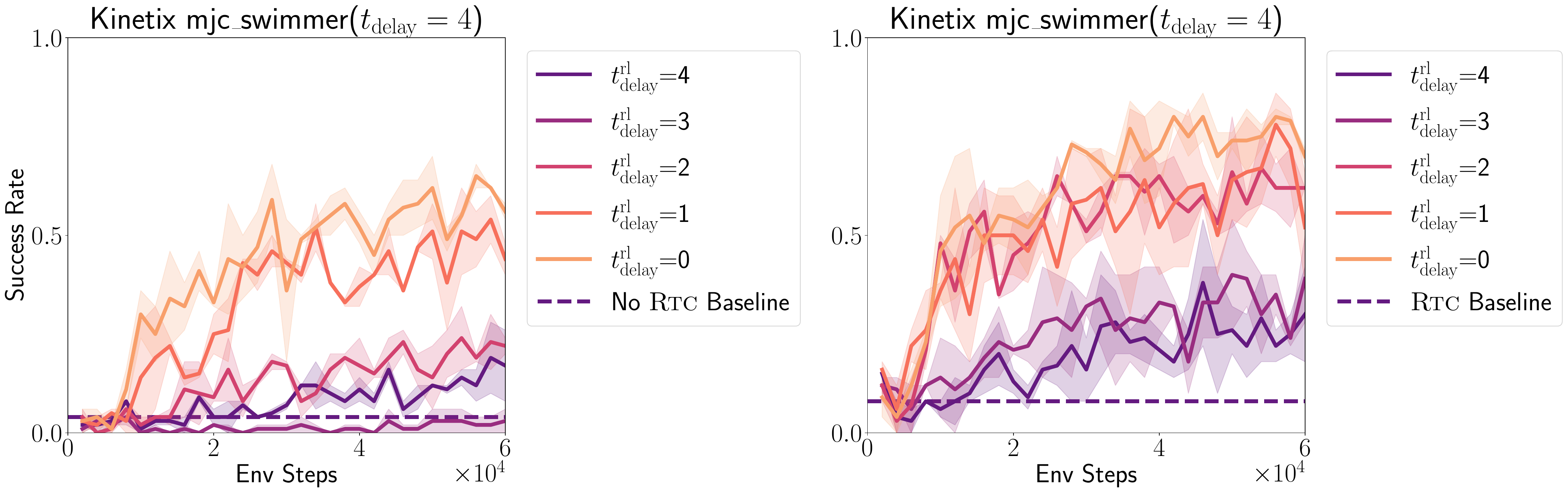}
    \caption{\arli performance at varying RL inference delays on \texttt{mjc\_swimmer}. \textit{Left}: without \rtc. \textit{Right}: with \rtc.}
    \label{fig:rl_infernce_delay_all}
\end{figure*}

We run additional ablations from Section \ref{sec:ablations} for the following questions.

\textbf{How critical is it that we condition $\pirl$ on the intermediate state and actions?} We run the same ablation on all tasks from the main results from Section \ref{sec:sim_experiments} (\texttt{mjc\_swimmer}, \texttt{mjc\_walker}, \texttt{car\_launch}, \alohacube) and also ablate between running \arli with and without \rtc. Figure \ref{fig:state_repr_naive_all} shows ablation results without \rtc and Figure \ref{fig:state_repr_rtc_all} shows ablation results with \rtc. Adding \rtc reduces the amount of additional information that needs to be added to final state representation, but across all tasks and environments we see that including all intermediate information consistently produces the best results.

\textbf{How sensitive is \arli to $\tdelayrl$?} We run the same set of ablations over $\tdelayrl$ on \texttt{mjc\_swimmer} with and without \rtc, Figure \ref{fig:rl_infernce_delay_all} shows the results of the experiments. Here we see that removing \rtc increases the requirement for how up-to-date the intermediate state must be. Between \rtc and no \rtc we see that significant improvements in policy performance only appear when $\tdelayrl < 3$ and $\tdelayrl < 2$, respectively.

\clearpage

\section{Learning Dynamics of \arli Noise Steering}

To better understand the learning behaviour of \arli during training, we analyze the learned \arli policy on the real-world Shoe-in-Bag task over four held-out episodes not seen from the model during training.

In this diagnostic, we focus on the \arli variant conditioned on previous actions only, since the considered held-out trajectories saved the previous-actions information, but not the mid-inference observation; we therefore use the previous-action model so that every checkpoint is evaluated on exactly the same modalities used for training. For Shoe-in-Bag, the learning curves of full \arli and the previous-action-only variant are similar, making this a useful diagnostic for the learned steering behavior. We replay two successful and two failed held-out Shoe-in-Bag episodes through every trained \arli checkpoint. For each replayed inference step, we record the mean and variance of the noise distribution predicted by \arli before it is denoised by the action expert. The first checkpoint considered is after the 24000 offline training steps, corresponding to the first model trained after collecting the initial warm-up episodes in the replay buffer.

The middle column of Figure~\ref{fig:arli_analysis} visualizes how the predicted mean and variance of the policy’s learned noise distribution evolve across training checkpoints for the replayed trajectories. The main cross-episode trend is that failed episodes exhibit a directional shift in the predicted noise distribution that is not present in the same way for successful episodes. On the held-out failures, later \arli checkpoints produce a more displaced and somewhat narrower steering-noise distribution: the predicted noise mean moves farther from zero while the predicted variance decreases. This pattern is consistent with later \arli checkpoints becoming more confident and more biased in particular regions of the failed trajectories.

This shift is localized in time rather than uniformly spread over the whole trajectory. In the heatmaps, the increased mean magnitude and reduced variance are most visible around the phase that requires the largest corrective steering. The first successful episode is smooth and well aligned, so there is little need for \arli to apply a large steering correction; accordingly, it shows the weakest localized signal among the four examples. The second successful episode is still successful but less smooth, which is consistent with the stronger mid-trajectory signal visible in the plots. The first failed episode is close to success: the shoe is only slightly misaligned with the bag opening, but the remaining correction is difficult and ultimately leads to failure. This helps explain why the signal for this episode becomes more pronounced later in training, when the learned steering distribution has become more specialized. The second failed episode represents the typical failure mode of this task, often observed in early failed episodes, where the shoe is substantially misaligned with the bag entrance. In this case, the learned signal is visible earlier in the checkpoint sequence, suggesting that this failure mode is easier for \arli to identify.

The latent-dimension view in the right column of Figure~\ref{fig:arli_analysis} provides a coarse diagnostic of where the checkpoint-dependent changes occur in the learned noise space. To quantify this, we compute the symmetric KL contribution of each latent dimension between the first and final checkpoints, averaged over inference steps. Overall, the drift remains broadly distributed across the latent dimensions, with effective dimension counts between 22.7 and 26.7 out of 32, suggesting that the learned change is not concentrated in only one or two specific coordinates. Failed 2 shows the clearest, though still moderate, concentration: its top three latent dimensions explain 25.3\% of the total early-to-late KL contribution, compared with roughly 19--20\% for the other episodes. We can therefore interpret this mainly as evidence that the learned change is distributed across the latent noise space, with only limited dimension-level localization.

Overall, this analysis suggests that \arli does not merely rescale the pretrained policy's noise input globally, rather it reshapes the noise distribution in a trajectory- and phase-dependent manner. The considered successful episodes show weak localized distributional-drift, while the failed episodes are distinguished by the combination of three effects: the predicted noise mean moves farther from zero, the predicted variance decreases more strongly, and the largest distributional changes concentrate around replay steps corresponding to the middle alignment/insertion portion of the task.

This is consistent with the qualitative observation that failures in Shoe-in-Bag often occur around alignment and insertion rather than at the very beginning of the episode. The findings also align with the role of \arli's intermediate information: because it observes the actions that will be executed during inference and, in the full method, can also use a more recent intermediate state, it can learn steering corrections that are targeted to the effective state at which the next action chunk will be used.

We emphasize that this analysis is intentionally conservative: since it uses only four held-out trajectories replayed through multiple checkpoints of a single task, we interpret it as descriptive evidence about how the learned noise distribution evolves on these episodes, not as a population-level success/failure classifier.

\begin{figure}
    \centering 
    \includegraphics[width=1\linewidth]{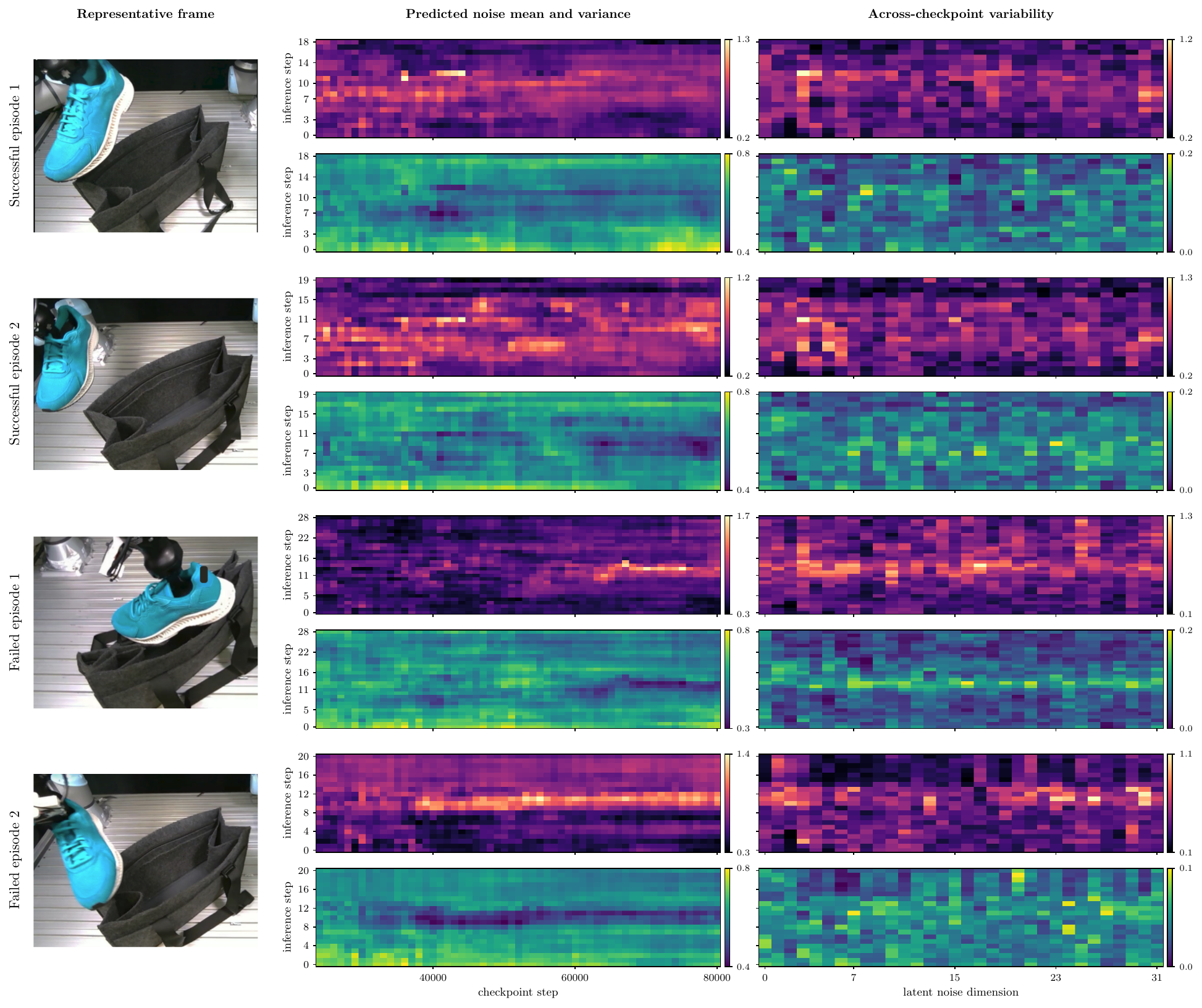}
    \vspace{-1.5em}
    \caption{Checkpoint-wise evolution of the \arli (actions only) noise distribution on four held-out Shoe-in-Bag episodes. \textit{Left}: a representative frame from the middle of each trajectory. \textit{Middle}: predicted noise mean (upper plot) and variance (lower plot) by inference step and checkpoint, averaged over latent noise dimensions. \textit{Right}: across-checkpoint standard deviation of the predicted noise mean (upper plot) and variance (lower plot) by inference step and latent noise dimension.
    }
    \label{fig:arli_analysis}
    \vspace{-1em}
\end{figure}

\newpage

\section{Theoretical Results}
\label{appdix:theory}

Our main theoretical result (Proposition \ref{prop:delay_subopt}) is built on top of the framework in \citet{li2025decoupled}. 

We first start by describing all the existing assumptions (\cref{assumption:data} from \citet{li2025decoupled}) and a crucial new assumption (\cref{def:delay-oracle}) that quantifies the sub-optimality caused by inference delays.

\input{body/theory/statement}

\input{body/theory/proof}

%% file: body/theory/statement.tex
\begin{assumption}[Data Obeys the Transition Dynamics]
\label{assumption:data}
$\mathcal{D} \in \Delta_{\mathcal{T}}$ is a trajectory distribution generated by rolling out a behavior policy from a distribution of $s_t \sim \mu$. The behavior policy can be non-Markovian (\emph{i.e.}, $\pi_\beta(a_{t+h} \mid s_{t:t+h+1}, a_{t:t+h})$). Each subsequent state is generated according to the dynamics of the MDP $\mathcal{M}$: $s_{t+h+1} \sim P(\cdot \mid s_{t+h}, a_{t+h}), \forall h \in \{0, 1, \cdots, k-1\}$.
The resulting trajectory is $(s_t, s_{t+1}, \cdots, s_{t+k}, a_t, a_{t+1}, \cdots, a_{t+k}) \in \mathcal{T} = \mathcal{S}^{k} \times \mathcal{A}^k$. 
\vspace{1mm}
\end{assumption}

\begin{definition}[Strong Open-Loop Consistency]
\label{def:olc}
$\mathcal{D}$ is strongly open-loop consistent if for every $a_{t:t+k} \in \mathrm{supp}(P_\mathcal{D}(a_{t:t+k} \mid  s_t))$, 
\begin{align}
\label{eq:strongoc}
    P(s_{t+k'} \mid s_t, a_{t:t+k'}) =  P_{\mathcal{D}}(s_{t+k'} \mid s_t, a_{t:t+k}), \quad \forall k' \in \{1, 2, \cdots, k\},
\end{align}
where we use $P(s_{t+k'} \mid s_t, a_{t:t+k'})$ to denote the distribution of the future state $s_{t+k'}$ after carrying out the action sequence $a_{t:t+k'}$ in the environment open-loop from $s_t$. 
\end{definition}

\textbf{Notations.}
Following the notations introduced by \citet{li2025decoupled}, we learn an action chunking Q-function, $\hat{Q}^+_{\mathrm{ac}}(s_t, a_{t:t+k})$. In our setting, there is a $d$-time-step delay on the observation $s_t$, so we may only react to the observation information that is $d$-step delayed (\emph{i.e.}, $s_{t-d}$) when carrying out $a_{t}$. Our policy approximates the best action chunk ($a_{t-d:t-d+k}$) that retains the same ``action prefix'' ($a^\bullet_{t-d}, \cdots, a^\bullet_{t-1}$) while maximizing $\hat{Q}^+_{\mathrm{ac}}(s_{t-d}, a_{t-d:t-d+k})$. We formalize this as follows:
\begin{align}
\label{eq:delayed-policy1}
    a^\circ_{t:t+k-d} = {\arg\max}_{\hat{a}_{t:t+k-d}} \hat{Q}^+_{\mathrm{ac}}(s_{t-d}, \hat a_{t-d:t+k-d}),
\end{align}
where 
\begin{align}
\label{eq:delayed-policy2}
    \hat{a}_{t-d:t+k-d} := \underbrace{a^\bullet_{t-d}, a^\bullet_{t-d+1}, \cdots, a^\bullet_{t-1}}_{\text{not used due to time delay}}, \underbrace{\hat{a}_t, \hat{a}_{t+1}, \cdots \hat{a}_{t+{k-d-1}}}_{\text{open-loop execution}}
\end{align}
We denote the policy that produces $a^{\circ}$ the $d$-delayed action chunking policy:
\begin{align}
    \pi_d(s_{t-d}) := a^{\circ}_{t:t+k-d}
\end{align}

\begin{lemma}[Open-loop consistency of $d$-delayed policy]
\label{lemma:olc-ddp}
If $\mathcal{D}$ satisfies \cref{assumption:data} and it is also collected by rolling out an $d$-delayed action chunking policy for any $d \geq 0$, that is
\begin{align}
    P_{\mathcal{D}}(a_{t+d:t+k} \mid s_{t:t+k}) = P_{\mathcal{D}}(a_{t+d:t+k} \mid s_t) = \delta_{\pi(s_t)},
\end{align}
then $\mathcal{D}$ is open-loop consistent.
\end{lemma}
\begin{proof}
$P_{\mathcal{D}}(a_{t+d:t+k} \mid s_{t:t+k}) = P_{\mathcal{D}}(a_{t+d:t+k} \mid s_t)$ implies that the action chunk $a_{t+d:t+k}$ is independent from the intermediate actions $s_{t+1:t+k}$. In addition, we also know that $a_{t:t+d}$ have already been determined before $s_t$ and cannot be changed after observing $s_{t:t+k}$ either. Therefore, we have
\begin{align}
    P_{\mathcal{D}}(a_{t:t+k} \mid s_{t:t+k}) = P_{\mathcal{D}}(a_{t:t+k} \mid s_{t}).
\end{align}
This allows us to establish
\begin{align}
    P_{\mathcal{D}}(s_{t+k'} \mid s_t, a_{t:t+k}) = P(s_{t+k'} \mid s_t, a_{t:t+k}), \quad \forall k'\in \{1, 2, \dots, k\}
\end{align}
as desired.
\end{proof}

\begin{lemma}[Action chunking Q-learning is optimal under $d$-delayed action chunking policy data collection]
\label{lemma:d-optimal}
Let $\mathcal{D}$ be some data distribution collected by a mixture of $d$-delayed policies, and $a^\star_{t:t+k} \in \mathrm{supp}(P_\mathcal{D}(a_{t:t+k} \mid s_t))$ for $a^\star_{t:t+k} = {\arg\max}_{a_{t:t+k}} Q^\star_{\mathrm{ac}}(s_t, a_{t:t+k})$. Then
\begin{align}
    \hat{Q}^+_{\mathrm{ac}}(s_t, a_{t:t+k}) = Q^\star_{\mathrm{ac}}(s_{t}
, a_{t:t+k}), \quad \forall s_t, a_{t:t+k} \in \mathrm{supp}(P_{\mathcal{D}}),
\end{align}
where $\hat{Q}^+_{\mathrm{ac}}$ is the learned value of the non-delayed action chunking policy: $\pi^+_{\mathrm{ac}}:  s_t \mapsto \arg\max_{a_{t:t+k}} \hat{Q}^+_{\mathrm{ac}}(s_t, a_{t:t+k})$, and $Q^\star_{\mathrm{ac}}$ is the optimal value achievable by a non-delayed action chunking policy.
\end{lemma}
\begin{proof}
We first observe that any mixture of open-loop consistent data distribution is also open-loop consistent. It is not hard to conclude that $\mathcal{D}$ is open-loop consistent due to \cref{lemma:olc-ddp}. Now, we can reuse Theorem 1 from \citet{li2025decoupled} to conclude that
\begin{align}
    \hat{Q}_{\mathrm{ac}}^+(s_t, a_{t:t+k}) = Q^+_{\mathrm{ac}}(s_t, a_{t:t+k}),  \quad \forall s_t, a_{t:t+k} \in \mathrm{supp}(P_\mathcal{D}),
\end{align}
where $Q^+_{\mathrm{ac}}$ is the true value of the non-delayed action chunking policy $\pi^+_{\mathrm{ac}}$. 

Now, since we have assumed that the optimal action chunking policy data is in the support, the value maximizing action chunk must be the optimal action chunk. We can now conclude that 
\begin{align}
    \hat{Q}^+_{\mathrm{ac}}(s_t, a_{t:t+k}) = Q^\star_{\mathrm{ac}}(s_{t}
, a_{t:t+k}), \quad \forall s_t, a_{t:t+k} \in \mathrm{supp}(P_{\mathcal{D}})
\end{align}
as desired.
\end{proof}

The implication of this result is that we can use the data collected by any $d$-delayed action chunking policy to learn optimal action chunking Q-function as long as the data covers some optimal action chunking policy's behavior. However, even with the optimal action chunking Q-function, such optimal action chunking policy is not achievable because we have a delay of $d$ time steps in our setting. 

\begin{definition}[Delayed Oracle Optimality Gap]
\label{def:delay-oracle}
An MDP $\mathcal{M}$ exhibits $\omega_d$-delayed optimality gap if for any $s_t, a_{t:t+k}$, 
\begin{align}
\left|Q^\star_{\mathrm{ac}}(s_t, a_{t:t+k}) - \mathbb{E}_{P(\cdot \mid s_{t}, a_{t:t+k})}[R_{t:t+k-d} + \gamma^{h-d} V^\star_{\mathrm{ac}}(s_{t+k-d}, a_{t+k-d:t+k}) \right| \leq \omega_d,
\end{align}
where $V^\star_{\mathrm{ac}}(s_t, a_{t:t+d}) = {\max_{a_{t+d:t+k}}} Q^\star_{\mathrm{ac}}(s_t, a_{t:t+k})$.
\end{definition}
Intuitively, $\omega_d$ upper-bounds the amount of sub-optimality resulted from a delayed decision. The left-hand-side (LHS) of the inequality quantifies the maximum value achievable when the decision is not delayed and at a regular interval of $k$ (\emph{i.e.}, make decision at $s_{t}, s_{t+k}, \cdots$). The right-hand-side (RHS) of the inequality quantifies the maximum value achievable when the decision is delayed (\emph{i.e.}, $a_{t+k:t+2k-d}$ is based on the delayed $s_t$) and because of the delay, the prefix of the action chunk is pre-determined (\emph{i.e.}, $a_{t+k-d:t+k}$).

\begin{theorem}[Delayed Action Chunking Policy is Near-optimal]
Let $\mathcal{D}$ be some data distribution colleted by a mixture of $d$-delayed policies, and the following two support assumptions hold:
\begin{enumerate}
    \item Optimal action chunks are in-support:
    \begin{align}
    \label{eq:supp-as1}
        a^\star_{t:t+k} \in \mathrm{supp}(P_{\mathcal{D}}(\cdot \mid s_t)),
    \end{align}
    where $a^\star_{t:t+k} = \pi^\star_{\mathrm{ac}}(s_t)$.
    \item Optimal action chunk \emph{completions} for in-support delayed action prefix are also in-support:
    \begin{align}
    \label{eq:supp-as2}
        a^\star_{t+d:t+k} \in \mathrm{supp}(P_{\mathcal{D}}(\cdot \mid s_t, a_{t:t+d})), \quad \forall a_{t:t+d} \in \mathrm{supp}(P_{\mathcal{D}}(\cdot \mid s_t)),
    \end{align}
    where $a^\star_{t+d:t+k} = {\arg\max}_{a_{t+d:t+k}} Q^\star_{\mathrm{ac}}(s_t, a_{t:t+k})$.
\end{enumerate}
If additionally, the MDP exhibits $\omega_d$-delayed optimality gap, then the value of the $d$-delayed action chunking policy $\pi^+_d$, $V^+_d$ satisfies
\begin{align}
    |V^+_d(s_t, a_{t:t+d}) - V^\star_{\mathrm{ac}}(s_t, a_{t:t+d}) | \leq \frac{\omega_d}{1-\gamma^{k-d}}, \quad \forall s_t, a_{t:t+d} \in \mathrm{supp}(P_{\mathcal{D}}),
\end{align}
where $V^\star_{\mathrm{ac}}(s_t, a_{t:t+d}) := {\max}_{a_{t+d:t+k}} Q^\star_{\mathrm{ac}}(s_t, a_{t:t+k})$ and $V^+_d(s_t, a_{t:t
+d}):= Q^+_d(s_t, \tilde a_{t:t+k})$ with $\tilde a_{t:t+k} = [a_{t:t+d}, a^\star_{t+d:t+k}]$, $a^\star_{t+d:t+k} = {\arg\max}_{a_{t+d:t+k}} Q^+_{\mathrm{ac}}(s_t, a_{t:t+k})$, and $Q^+_d$ is the fixed point of the following bellman equation:
\begin{align}
    Q^+_d(s_t, a_{t:t+k}) = \mathbb{E}_{P(\cdot \mid s_{t}, a_{t:t+k})}[R_{t:t+k-d} + \gamma^{k-d} Q^+_d(s_{t+k-d}, \tilde a_{t+k-d:t+2k-d})],
\end{align}
where again $\tilde a_{t+k-d:t+2k-d} := [a_{t+k-d:t+k}, a^\star_{t+k:t+2k-d}]$ with $a^\star_{t+k:t+2k-d} = {\arg\max}_{a_{t+k:t+2k-d}} Q^+_{\mathrm{ac}}(s_{t+k-d}, a_{t+k-d:t+2k-d})$
\end{theorem}

\begin{proof}
We first note that due to the first support assumption (\cref{eq:supp-as1}), we can trigger \cref{lemma:d-optimal} to conclude that $\hat{Q}^+_{\mathrm{ac}}(s_t, a_{t:t+k}) = Q^\star_{\mathrm{ac}}(s_t, a_{t:t+k})$ for in-support $s_t, a_{t:t+k}$. On top of that, due to the second support assumption (\cref{eq:supp-as2}), we know that the optimal action chunk completion would also be in-support, which means that
\begin{align}
    {\arg\max}_{a_{t+d:t+k}} \hat{Q}^+_{\mathrm{ac}}(s_t, a_{t:t
    +k}) = {\arg\max}_{a_{t+d:t+k}} Q^\star_{\mathrm{ac}}(s_t, a_{t:t
    +k}), \quad \forall s_t, a_{t:t+d} \in \mathrm{supp}(P_{\mathcal{D}}),
\end{align}
so we can use $a^\star_{t+d:t+k}$ to denote the optimal action chunk completion for both value function.

We first write out
\begin{align}
    Q^+_d(s_t, a_{t:t+k}) = \mathbb{E}_{P(\cdot \mid s_{t}, a_{t:t+k})}[R_{t:t+k-d} + \gamma^{k-d} Q^+_d(s_{t+k-d}, [a_{t+k-d:t+k}, a^\star_{t+k:t+2k-d}])],
\end{align}
and compare with 
\begin{align}
    X = \mathbb{E}_{P(\cdot \mid s_{t}, a_{t:t+k})}[R_{t:t+k-d} + \gamma^{k-d} Q^\star_{\mathrm{ac}}(s_{t+k-d}, [a_{t+k-d:t+k}, a^\star_{t+k:t+2k-d}])].
\end{align}
From the $\omega_d$-delayed optimality assumption, we know that
\begin{align}
    |X - Q^\star_{\mathrm{ac}}(s_t, a_{t:t+k})| \leq \omega_d.
\end{align}

Set $\tilde a_{t+k-d:t+2k-d} = [a_{t+k-d:t+k}, a^\star_{t+k:t+2k-d}]$. We can now bound the difference between $Q^+_d$ and $Q^\star_{\mathrm{ac}}(s_t, a_{t:t+k})$.
\begin{align}
    &|Q^+_d(s_t, a_{t:t+k}) - Q^\star_{\mathrm{ac}}(s_t, a_{t:t+k})| \\
    & \quad \leq \mathbb{E}_{P(\cdot \mid s_{t}, a_{t:t+k})}[\omega_d + \gamma^{k-d}|Q^+_d(s_{t+k-d}, \tilde a_{t+k-d:t+2k-d}) - Q^\star_{\mathrm{ac}}(s_{t+k-d}, \tilde a_{t+k-d:t+2k-d})|] \\
    &\quad \leq \frac{\omega_d}{1-\gamma^{k-d}}.
\end{align}
This means that
\begin{align}
    |V^+_d(s_t, a_{t:t+d}) - V^\star_{\mathrm{ac}}(s_t, a_{t:t+d})| \leq \frac{\omega_d}{1-\gamma^{k-d}}, \quad \forall s_t, a_{t:t+d} \in \mathrm{supp}(P_{\mathcal{D}}).
\end{align}

\end{proof}